\documentclass{article}

\usepackage[preprint]{neurips_2024}

\usepackage[utf8]{inputenc} 
\usepackage[T1]{fontenc}    
\usepackage{hyperref}       
\usepackage{url}            
\usepackage{booktabs}       
\usepackage{amsfonts}       
\usepackage{nicefrac}       
\usepackage{microtype}      
\usepackage[table]{xcolor}
\usepackage{microtype}
\usepackage{graphicx}
\usepackage{subfigure}

\usepackage{booktabs} 

\usepackage{amsmath}
\usepackage{amssymb}
\usepackage{mathtools}
\usepackage{amsthm}

\usepackage[capitalize,noabbrev]{cleveref}

\theoremstyle{plain}
\newtheorem{theorem}{Theorem}[section]
\newtheorem{proposition}[theorem]{Proposition}

\theoremstyle{definition}

\theoremstyle{remark}

\usepackage{multirow}
\usepackage[textsize=tiny]{todonotes}
\usepackage{paralist}

\title{BiEquiFormer: Bi-Equivariant Representations for Global Point Cloud Registration}

\author{%
    Stefanos Pertigkiozoglou \thanks{Equal Contribution}\\
    University of Pennsylvania\\
    \texttt{pstefano@seas.upenn.edu}\\
    \And 
    Evangelos Chatzipantazis \footnotemark[1] \\
    University of Pennsylvania\\
    \texttt{vaghat@seas.upenn.edu}\\
    \And 
    Kostas Daniilidis \\
    University of Pennsylvania \\
    Archimedes, Athena RC\\
    \texttt{kostas@cis.upenn.edu}\\
}

\definecolor{tabfirst}{rgb}{1, 0.7, 0.7} 
\definecolor{tabsecond}{rgb}{1, 0.85, 0.7} 
\definecolor{tabthird}{rgb}{1, 1, 0.7} 

\begin{document}

\maketitle

\begin{abstract}
  The goal of this paper is to address the problem of \textit{global} point cloud registration (PCR) i.e., finding the optimal alignment between point clouds irrespective of the initial poses of the scans. This problem is notoriously challenging for classical optimization methods due to computational constraints. First, we show that state-of-the-art deep learning methods suffer from huge performance degradation when the point clouds are arbitrarily placed in space. We propose that \textit{equivariant deep learning} should be utilized for solving this task and we characterize the specific type of bi-equivariance of PCR. Then, we design BiEquiformer a novel and scalable \textit{bi-equivariant} pipeline i.e. equivariant to the independent transformations of the input point clouds. While a naive approach would process the point clouds independently we design expressive bi-equivariant layers that fuse the information from both point clouds. This allows us to extract high-quality superpoint correspondences and in turn, robust point-cloud registration. Extensive comparisons against state-of-the-art methods show that our method achieves comparable performance in the canonical setting and superior performance in the robust setting in both the 3DMatch and the challenging low-overlap 3DLoMatch dataset.
\end{abstract}

\section{Introduction}
Point Cloud Registration (PCR) is at the frontend of many robotics and vision pipelines. The goal, in the pairwise and rigid setting, is to align two partially overlapped point clouds expressed in their own coordinate system by estimating a roto-translation between them and fusing them in a common coordinate system.  It has been successfully applied in many tasks such as 3D Scene Reconstruction \cite{Blais1995},  SLAM \citep{Nuchter2006} and pose estimation \cite{Yang2013}. 

While PCR has been studied extensively over the past decades, the desiderata for real-time and robust registration of real-world applications makes the problem extremely challenging. Especially in environments with repetitive patterns such as indoor environments as well as in low-overlap settings that appear loop closure tasks \cite{Bosse2008} the requirement for distinctive point-wise features for correspondence is enhanced. A particularly challenging aspect of the problem is the robustness w.r.t. the initial poses of the point clouds. In classical optimization methods, the problem is called \textit{global} PCR and is famously intractable due to the large volume of points \cite{Yang2013}. 


Deep learning has been proven very effective in PCR in all building blocks of the registration pipeline. Powerful point cloud architectures \cite{qi2016pointnet, Thomas2019KPConv} serve both as the feature extraction for correspondence-based methods \cite{zeng20163dmatch,FCGF2019} and a way to identify distinctive features for matching \cite{Huang_2020_CVPR,lepard2021}. It has also been utilized to learn robust estimators \cite{Choy_2020_CVPR,pais19,bai2021pointdsc} or directly regress the relative transformation \citep{Wang_2019_ICCV,yaoki2019pointnetlk}. However, as we show next, the problem of \textit{global} PCR is not correctly characterized and still remains unsolved. 

In this work, we show how recent state-of-the-art registration pipelines are heavily affected by the orientations of the initial scans, especially in challenging low-overlap settings (Fig. \ref{fig:brokenall}). Subsequently, we propose BiEquiformer a detector-free attention pipeline that is bi-equivariant to the roto-translation group (Fig.\ref{fig:main}). Our main contributions can be summarized as follows:
\begin{figure*}
    \centering
    \includegraphics[width=0.85\textwidth]{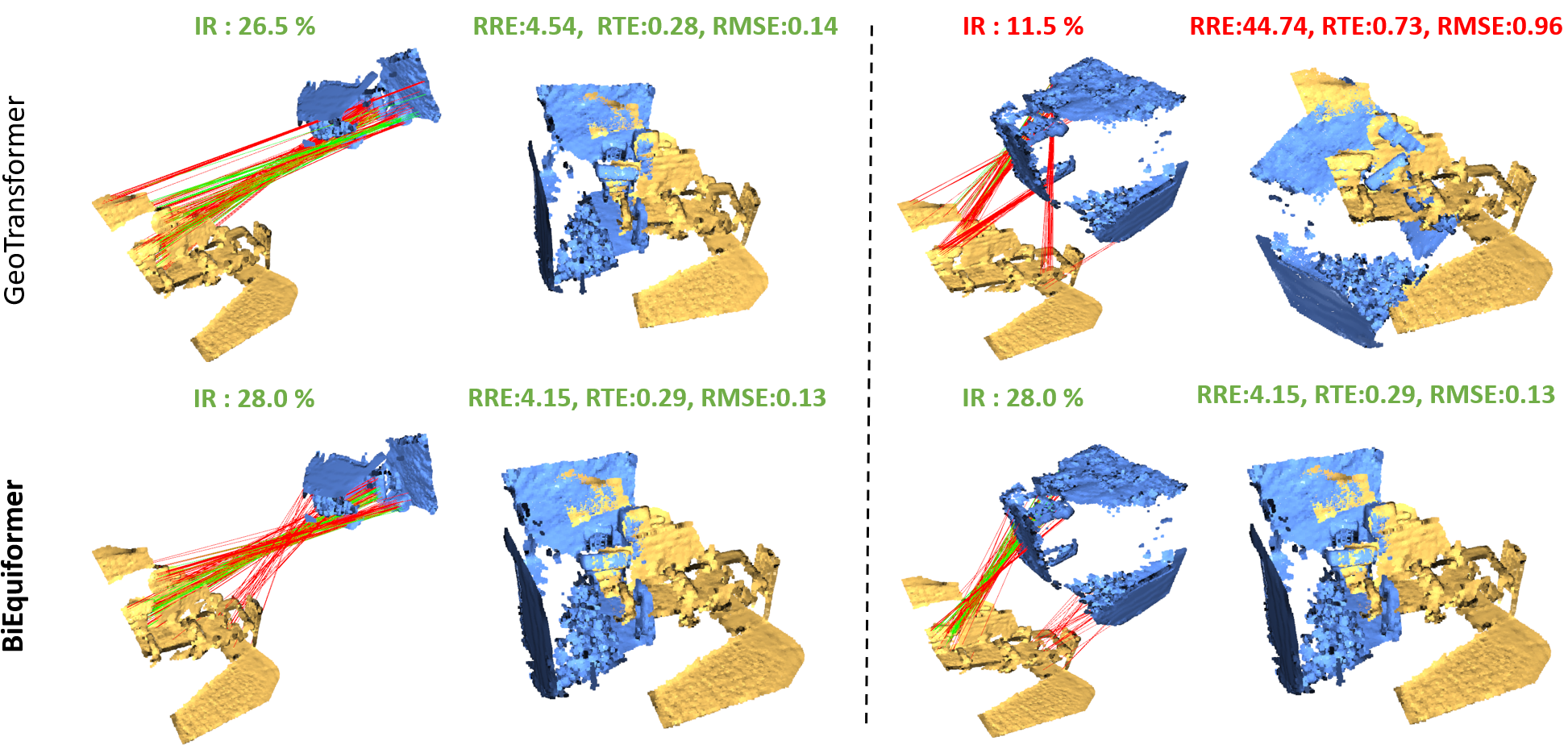}
    \caption{Inlier Ratios (IR) and Registration Metrics (RRE,RTE,RMSE) for two pairs of low-overlap scans that differ only by their relative pose. (Left) Both GeoTransformer (a state-of-the-art method)  and BiEquiformer recover the correct registration and high IR. (Right): GeoTransformer fails to find good matches (low IR) in this relative pose and predicts an incorrect registration. In contrast, BiEquiformer is designed to perform consistently irrespective of the initial point cloud poses.}
    \vskip -0.6cm
    \label{fig:brokenvisual}
\end{figure*}
\begin{compactenum}
    \item \textbf{The state of Global PCR in DL}: We investigate the robustness of state-of-the-art methods under rigid transformations of the input point clouds. In Fig. \ref{fig:brokenall} we show that in numerous popular state-of-the-art methods there is a deterioration in performance when the initial poses of the point clouds vary, exacerbated as the overlap between scans becomes smaller. A visual example can be seen in Figure \ref{fig:brokenvisual} and is discussed in detail in Section \ref{sec:Rob}.
    \item \textbf{Bi-Equivariance and PCR}: We formulate and characterize the specific bi-equivariance properties of PCR (Section \ref{sec:PrFormulation}). Then we propose novel layers that process invariant, equivariant, and different types of bi-equivariant features, which extend standard equivariant layers by fusing information between the point clouds (Section \ref{sec:method}). 
    \item \textbf{State-of-the-art in Global PCR}: Combining those layers we propose a novel, scalable equivariant pipeline for point cloud registration. Our method ensures consistent registration results, regardless of the initial configuration of the input point clouds, and achieves state-of-the-art registration accuracy in the robust setting, especially in low-overlap datasets.
\end{compactenum}

\section{Related Work}
\vskip -0.2cm
Point cloud registration (PCR) is a fundamental problem with extensive literature. Here we focus on related work on rigid geometric PCR i.e., the point clouds can be aligned with a roto-translation, and only depth is provided without any other exterior information such as color, etc.\\
\textbf{Classic Methods; ICP and Global Registration}. Over the previous decades, various methods have been proposed. Extensive surveys \citep{Pomerleau, Bellekens2015, Li2021ATR} categorize and benchmark classical algorithms or main building blocks of those e.g., the local feature extraction backbone \cite{3DLocalFD} or the robust estimators \cite{Babin2018AnalysisOR}. Stemming from the pioneering papers that introduced the Iterative Closest Point (ICP) algorithm \cite{Chen1991, Besl1992}, a number of variants have been proposed \cite{Pomerleau}. 
The non-convexity of PCR with unknown correspondences makes ICP susceptible to local optima and usually, a relatively accurate initial registration has to be provided. This initiated the problem of \textit{Global} PCR where methods treat PCR as a global optimization problem \citep{li2007,Yang2013} and solve it using Branch and Bound or more recently, graduated non-convexity \citep{Yang2021Teaser,qiao_pagor}. It is common to use such methods only as an initial estimate for registration that is subsequently refined by ICP. These methods usually run in exponential time thus facing scalability issues in scene-level scans. Our goal in this paper is to design a \textit{global} PCR method that is scalable and robust even in low-overlapping settings. \\ 
\textbf{Local 3D Feature Descriptors:} Extracting descriptors from the point clouds 
to characterize the local geometry is a common building block of most registration pipelines. 
Earlier works extract hand-designed features in the form of histograms \cite{Rusu2008Al} that encode the 3D spatial distribution of points \cite{Johnson1999}, the orientations of the neighbors \cite{SALTI2014251, Makadia2006Reg} or the differences with the neighbors in the Darboux frame \cite{Rusu2009FPFH}. More recently, deep learning architectures on point clouds \citep{dgcnn,qi2016pointnet,Qi2017PNplus,Thomas2019KPConv,FCGF2019} has been utilized for end-to-end feature extraction either by using MLPs to compactify hand-designed features \cite{Gojcic2018LEARNEDCL}, 3D convolutional networks \cite{zeng20163dmatch, FCGF2019} to encode local volumetric patches or utilizing Transformers \cite{vaswani} to encode both global and local context within and between the point clouds \cite{Huang_2021_CVPR, qin2022geometric}. While hand-designed descriptors have the advantage of being data-agnostic, they are susceptible to noise and occlusions. \\
\textbf{Correspondence-Based PCR:} Correspondence-based methods utilize the local descriptors in order to match points or surfaces between the points clouds before estimating the transformation. They are split between keypoint-based methods, that explicitly search for a small subset of distinctive features to perform the matching, or detector-free methods that perform a dense matching of points accounting for the outliers too. In the former category, in the deep learning literature the pioneering work of 3DMatch \cite{zeng20163dmatch} was followed by many works that learn to match the learned keypoints \citep{yew2018-3dfeatnet,FCGF2019,sarode2019pcrnet,Deng2018PPFNetGC,gojcic20193DSmoothNet,bai2020d3feat, wang2022you,idam}. Predator \cite{Huang_2021_CVPR} proposed that not only saliency but proximity to the overlap region should be considered in the keypoint detection and proposed a novel self-attention/cross-attention pipeline to learn that. More recently, keypoint-free deep learning methods have been introduced that perform matching in a coarse-to-fine fashion \cite{yu2021cofinet,Min_2021_ICCV,yang2022one} and have shown increased performance and robustness in low overlap settings \cite{lepard2021,qi2016pointnet}.\\
\textbf{Equivariant Registration:} As a step towards \textit{global PCR}, \textit{equivariant deep learning} can be utilized. Currently, the issue of full 3D roto-translation invariance is not always treated properly by end-to-end learned descriptors. Most of the deep learning registration pipelines are not equivariant to the point cloud poses thus requiring a great amount of data augmentations \cite{qin2022geometric} while still behaving inconsistently during inference (Fig. \ref{fig:brokenall}). In this category, PPFNet \cite{Deng2018PPFNetGC, Deng2018PPFFoldNetUL} is a keypoint-based method that introduces hand-designed rotation-invariant point features as local descriptors. YOHO \cite{wang2022you} utilizes a feature extractor equivariant to the icosahedral group while SpinNet \cite{ao2020SpinNet} uses a cylindrical convolution to extract planar equivariant features. GeoTransformer \cite{qin2022geometric} takes a step forward by encoding pose invariant features in the superpoint transformer. However, the feature backbone is not rotation-equivariant. Powerful rotation equivariant networks that operate on point clouds have been proposed \cite{chen2021equivariant,deng2021vn,pointconvformer}. They have been successfully utilized in 3D Shape Reconstruction \cite{SE3Recon2023,chen2022equivariant}, Segmentation \cite{Deng2023BananaBF}, Protein-Docking \cite{ganea2021independent}, Robotic Manipulation \cite{ryu2023equivariant,Ryu_2024_CVPR},\cite{huang2024fourier} etc. Building on that successful usage of equivariant deep learning we propose a detector-free, transformer-based registration pipeline that is bi-equivariant to the independent roto-translations of both the source and reference point clouds.  \vskip -0.3cm

\begin{figure*}
    \centering
    \vskip -0.1in
    \includegraphics[height=0.2\textheight, width=\textwidth]{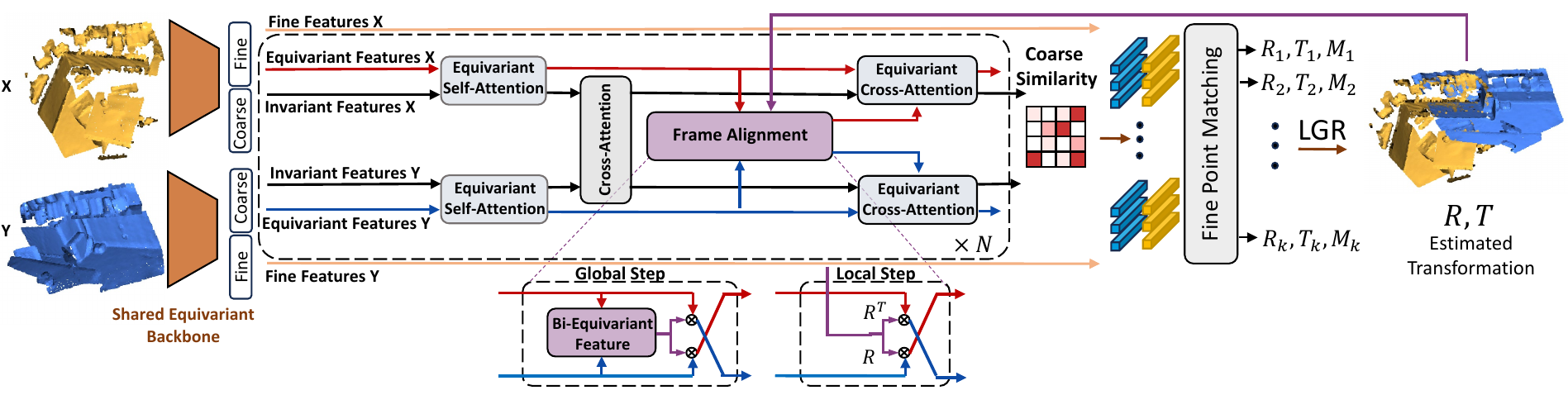}
    \caption{BiEquiFormer is an attention-based bi-equivariant pipeline for global PCR. First, equivariant intra-point self-attention and inter-point cross-attention layers update the scalar and vector features on the points. Then a bi-equivariant feature is used to align the input vectors to the same frame before applying equivariant cross-attention. 
    The output invariant coarse features are used to extract a set of candidate coarse matches which are processed by a fine point matching module to extract a candidate transformation. A final estimate is computed using a local-to-global transformation scheme. After the first transformation is estimated (Global Step) we can apply BiEquiFormer iteratively by switching the bi-equivariant frame alignment block with the current rotation estimation (Local Step).}
    \label{fig:main}
    \vskip -0.55cm
\end{figure*}

\section{Problem Formulation and Characterization of Equivariant Properties}\label{sec:PrFormulation}
\vskip -0.2cm

\par Consider two observers, the \textit{reference} and the \textit{source}, each with distinct coordinate frames $r$ and $s$ respectively, sampling points in their respective frames $X^r= \{x_i\in\mathbb{R}^3|i=1,\ldots,N\}$, $Y^s=\{y_j\in\mathbb{R}^3|j=1,\ldots,M\}$.  Let $\mathrm{SE(3)}$ denote the group of roto-translations and $\mathrm{SO(3)}$ its subgroup of rotations. The objective of point cloud registration (under the assumption of unique alignment) is to find the rigid transformation $\mathcal{T}_s^r \in \mathrm{SE(3)}$ that aligns the coordinate frame $s$ to $r$ using only the sampled points $X^r, Y^s$. Once the relative rotation and translation parameters $R_s^r \in SO(3),T_s^r \in \mathbb{R}^3$ that constitute $\mathcal{T}_s^r$, are estimated we can transform $Y^s$ to the reference coordinate frame and get $Y^r:= \mathcal{T}_s^rY^s  := R_s^r Y^s+T_s^r=\{R_s^r y+T_s^r\in \mathbb{R}^3|y\in Y^s\}$. This transformation allows the merging of the two observations through the union $X^r\cup Y^r$.

\par To solve this problem we assume that there exists an overlapping area of the surface sampled by both observers. Then, we can reduce PCR into a simultaneous correspondence and pose estimation problem. Specifically, we assume that there exists a subset $X_\mathrm{o}\subseteq X^r$ such that for every point $x_m\in X_\mathrm{o}$ there exists a corresponding $y_m\in Y^r:=R_s^rY^s+T_s^r$ such that $\lVert x_m-y_m\rVert\leq \epsilon$ for a small $\epsilon$. We refer to the points $x_i\in X_\mathrm{o}$ and their corresponding points $y_i\in Y^s$ as point matches. The goal is first to estimate these point matches. Given a set of such matching pairs $C=\{(x_i,y_i)|x_i\in X^r,y_i\in Y^s\}$, PCR estimates the relative transformation by solving the Procrustes optimization problem: ${\min}_{(R,T) \in SE(3)} \sum_{(x_i,y_i)\in C}\lVert Ry_i+T-x_i\rVert_2^2$. 

\textbf{Characterization of Equivariant Properties of PCR}:
To describe the geometric properties of the problem formally we need the notion of \textit{equivariance} first. Given a group $G$ acting on two sets $S_i, S_o$ via the (left) actions $*,\Tilde{*}:G \times S \rightarrow S$ (in our cases those sets will either be (sets of) vector spaces or a sub-group of $G$ where the action will be properly defined) a map $f: S_i \rightarrow S_o$ is \textit{equivariant} w.r.t. the group actions if for all $g \in G, s \in S_i$: $f(g*s)=g \Tilde{*} f(s)$. For clarity, we suppress $*,\Tilde{*}$ and simply write $gs$ for the group action of $G$ on $S$. The above formulation of PCR implies the following properties for the estimated transformations.\\
\textbf{$\mathrm{\mathbf{SE(3)}}$ Bi-Equivariance}: The transformation should remain consistent under (proper) rigid transformations of either $X^r$ or $Y^s$. Formally, we define a function $f: S_i \rightarrow S_o$ to be \textit{$\mathrm{SE(3)}$-bi-equivariant} (extends to any group $G$) if it is equivariant w.r.t. the \textbf{\textit{joint}} group action of the direct product group $\mathrm{SE(3)} \times \mathrm{SE(3)}$ defined as $s \mapsto g_1 * s \cdot g_2 ^{-1}$, where $*,\cdot$ are left and right group actions respectively that are jointly associative i.e. $g_1 * (s \cdot g_2 ^{-1}) = (g_1 * s) \cdot g_2 ^{-1}$ (In our case all actions are implemented using matrix multiplications which are both left and right associative; we omit the $*,\cdot$ to make notation more compact). We prove in Appendix Proposition \ref{prop:jointaction} that this \textbf{joint} action is a valid \textbf{left} action of the direct product group (whenever the actions $*, \cdot$ are well-defined). Depending on whether $s$ belongs to the domain $S_i$ or the co-domain $S_o$ of $f$ we define three cases (when the corresponding actions are well-defined) that we will use later. For all $(g_1, g_2) \in \mathrm{SE(3)} \times \mathrm{SE(3)}$:

\textbf{Output bi-equivariance}, $f:S_1 \times S_2 \rightarrow S_3$. $\forall (s_1,s_2) \in S_1 \times S_2$, $f(g_1 s_1, g_2 s_2) = g_1 f(s_1,s_2)g_2^{-1}.$\\
\textbf{Input bi-equivariance}: $f:S_1 \rightarrow S_2 \times S_3$, $f(s_1)=(s_2, s_3)$. $\forall s_1 \in S_1$, $f(g_1 s_1 g_2^{-1})= (g_1 s_2,g_2 s_3).$\\
\textbf{Input/Output bi-equivariance}: $f:S_1 \rightarrow S_2$ with $f(g_1 s_1 g_2^{-1}) = g_1 f(s_1)g_2^{-1}, \forall s_1 \in S_1$.

To analyze the properties of the PCR problem we will assume that the Procrustes optimization problem has a unique solution (a sufficient condition for that is that the set $C$ of matches includes 3 non-coplanar vectors). Given the overlap and optimality conditions discussed above, and under the additional assumption of unique ground-truth registration, PCR can be defined as a map $\bigcup_{N >0} \mathbb{R}^{3 \times N} \times \bigcup_{M >0} \mathbb{R}^{3 \times M} \rightarrow \mathrm{SE(3)}$ with $PCR(X^r, Y^s) = \mathcal{T}_s^r$. We can rigorously prove the following propositions using our definitions (all proofs in the Appendix).
\begin{proposition}\label{prop:pcr_bieq}
    PCR is output SE(3)-bi-equivariant. i.e. for all $(\mathcal{T}_1,\mathcal{T}_2) \in SE(3) \times SE(3)$: $PCR(\mathcal{T}_1X^r, \mathcal{T}_2Y^s) = \mathcal{T}_1\mathcal{T}_s^r \mathcal{T}_2^{-1}$.
\end{proposition}
\vskip -0.2cm 
\begin{proposition}\label{prop:eq_flip} (Reference-Source Interchangeability)
    PCR is equivariant to the ordering of the arguments. I.e. $C_2 = \{e,\frak{f}\}$ is the group of flips with $e$ the identity and $\frak{f}$ acting as: $f(X^r,Y^s) = (Y^s, X^r)$ then:
    $PCR(\frak{f}(X^r, Y^s)) = (\mathcal{T}_{s}^r)^{-1}.$
\end{proposition}
\begin{proposition}(Permutation Equivariance)\label{prop:eq_ord}
    PCR is invariant to the ordering of the points. I.e. if $S_N$ is the group of permutations of $N$ points: 
    $PCR(S_N X^r, S_M Y^s) = \mathcal{T}_s^r$.
\end{proposition}


\vspace{-14pt}
\section{Method}\label{sec:method}
\vskip -0.2cm
\subsection{Building Bi-equivariant feature maps}\label{sec:bilayers}
\vskip -0.2cm
While the literature is abundant with methods that build $\mathrm{SE(3)}$-equivariant representations there is a lack in the design of compact and expressive bi-equivariant feature maps i.e. feature maps that transform with the \textbf{joint action} of $\mathrm{SE(3)} \times \mathrm{SE(3)}$ as described in the previous section. This is particularly important in our problem since vanilla equivariant features do not fuse the information of both point clouds thus they create impoverished representations for matching. While the general theory from \cite{gentheory2019} can be adapted to find convolutional layers, such layers have a huge memory overhead and do not scale to scene-level scans.
Closer to our work, both \cite{ganea2021independent} and \cite{qin2022geometric} parametrize only the invariant channels when they fuse the features of the point clouds via cross-attention. However, useful vector features that can be learned on the points such as the normals of the surface cannot be represented this way.\\
We present some elementary operations on the feature maps that preserve bi-equivariance in each of the three cases above. In the next section, we utilize these operations to build expressive parametric layers for global PCR. We note that elementary equivariant operations such as the inner product or the norm of the difference, heavily used in equivariant literature, are \textbf{not} bi-equivariant. 
\begin{proposition}\label{prop:bifeats}\
    \begin{compactenum}
        \item If $f_1,f_2 \in \mathbb{R}^3$ are vector features i.e. they transform with the standard representation of $SO(3)$ then the tensor product $f_1, f_2 \mapsto f_1f_2^T$ is an SO(3) \textbf{output-bi-equivariant} map.
        \item Given a matrix $F \in \mathbb{R}^{3 \times 3}$ with distinct, positive, singular values that transforms with the joint action of $SO(3) \times SO(3)$ i.e., $F \mapsto R_1FR_2^T$ the map:
        $$F \mapsto (\{U_i\sigma(\Sigma)\}_{i=1}^4,\{V_i\sigma'(\Sigma)\}_{i=1}^4)$$ is an SO(3) \textbf{input-bi-equivariant map}, where $\{(U_i,\Sigma, V_i)\}_{i=1}^4$ are the 4 possible SVD decompositions of $F$ counting signs with $U_i,V_i \in SO(3)$ if $\det(F)>0$ and $U_i \in O(3)-SO(3),V_i \in SO(3)$ if $\det(F)<0$ and $\sigma, \sigma'$ are point-wise non-linearities on the singular values. 
        The $SO(3)$ matrices are formed as $[u_1, u_2, u_1 \times u_2]$ and the $O(3)-SO(3)$ as $[u_1, u_2, -u_1 \times u_2]$ where $u_1,u_2$ are the first two columns of $U$ (and similar for $V$).
        \item Given the same matrix $F \neq 0$ as above, the map 
        $F \mapsto \sigma(\|F\|) \frac{F}{\|F\|}$ is $\mathrm{SO(3)}$ \textbf{input-output bi-equivariant}, where $\|\cdot\|$ is a matrix norm e.g. operator, Frobenius, trace norm etc.
    \end{compactenum}
\end{proposition}

\textbf{Observation:} It is easy to verify that we can construct an SO(3)-equivariant map via the composition
$(\textit{\textbf{iBEq}} \circ (\circ_K \textit{\textbf{ioBEq}}^K) \circ \textit{\textbf{oBEq}})(X,Y)$ where \textit{\textbf{iBEq}}, \textit{\textbf{oBEq}}, $\textit{\textbf{ioBEq}}^K$ are SO(3) - input, output, and input-output bi-equivariant maps respectively. However, in contrast to standard equivariant layers, this composition fuses information from both inputs $X,Y$. This observation is crucial for our design.

\textbf{Architecture Overview:} We follow a coarse-to-fine approach similar to \citet{qin2022geometric}. The coarse superpoint matching stage estimates candidate pairs of matching point cloud patches (superpoints). Given these, the fine point matching stage estimates $R,T$ for the neighborhood of each candidate pair. Lastly, a local-to-global registration scheme (Appendix \ref{sec:loc2glob}), is used to evaluate each candidate transformation and select the highest-scoring one. Additionally, we propose that after the first estimated transformation (Global Step) an optional Local Refinement Step can be used, using only equivariant layers. To ensure bi-equivariance all parts of the pipeline must respect the constraint. We utilize VNN \cite{deng2021vn} as the feature extractor. In Appendix \ref{sec:featExtra} we describe an adaptation of VNN so that it processes both invariant $f_s$ and equivariant $f_v$ feature vectors and describe how we get the coarse $X_S,Y_S$ and fine points $X_D,Y_D$.

\vspace{-10pt}
\subsection{Invariant and Equivariant Attention Layers}\label{sec:selfattention}
\vskip -0.2cm
\textbf{Intra-Point Self Attention:} Assume we are given a point cloud $X$ along with its per-point equivariant and invariant features $f_s(x_i)$, $f_v(x_i)$. We propose an equivariant intra-point self-attention layer that can process both invariant and equivariant features. This layer is an extension of the invariant attention layer used in \cite{qin2022geometric} that is limited to use only invariant inputs. Specifically, we define the invariant and equivariant intra-point self-attention layers as follows:
\begin{align*}
    \alpha_s^\mathrm{intra}(x_i,f_s,f_v)=\sum_{x_j\in X}s_{ij}W_vf_s(x_j), \quad
    \alpha_v^\mathrm{intra}(x_i,f_s,f_v)=\sum_{x_j\in X}s_{ij}\mathrm{VN}_V(f_v(x_j))
\end{align*}
where $\mathrm{VN}_V$ is a learned Vector Neurons linear layer and $s_{ij} = \exp(e_{ij})/ \sum_{x_j'\in X}\exp(e_{ij'})$ where $e_{ij}$ is the attention score matrix defined as:
\begin{align*}
e_{ij} = \left(f_s(x_i)W_Q\right)\left(f_s(x_j)W_K+r_{ij}W_R\right)^T+ w_qf_v(x_i)^Tf_v(x_j)w^T_k
\end{align*}
with $r_{ij}$ being the invariant relative geometric embedding between $x_i,x_j$ introduced in \cite{qin2022geometric}, $W_Q, W_K$ being learned weight matrices and $w_q,w_k$ being learned weight vectors. In the Appendix Prop. \ref{prop:intra_inv} we prove the invariance of $\alpha_s^{\mathrm{intra}}$ and equivariance of $\alpha_v^{\mathrm{intra}}$.

\textbf{Equivariant Cross-Attention Layer:} The intra-point self-attention allows the exchange of information between points of the \textit{same} point cloud. Applying a similar mechanism for an inter-point cross-attention is not trivial when we want to use the equivariant features. That is because the two point clouds can rotate independently, and thus in order to combine these features we need a way to align them. We propose to do such an alignment by using a bi-equivariant feature extracted from a point pair that consists of a point transforming according to frame $r$ and a point transforming according to frame $s$. With this alignment, we can define an equivariant cross-attention layer that allows the exchange of information between the equivariant features of the two point clouds.

First, to define the point pair we assume a soft assignment $S_{XY}=\{s_{ij}\in [0,1]|\sum_{j=1}^{|Y|}s_{ij}=1, 0<i\leq |X|\}$ between the point clouds $X$ and $Y$ e.g. coming from the attention scores  $s_{ij}$ of a simple cross-attention layer that uses only the invariant features of the point clouds. Then for all $x_i\in X$ we compute the pairs $(x_i,y_{pi})$ where we define $y_{pi}=\sum_{j\in |Y|}s_{ij}y_j$ with features:
$$f_v(y_{pi})=\sum\nolimits_{j\in |Y|}s_{ij}f_v(y_j)\quad f_s(y_{pi})=\sum\nolimits_{j\in |Y|}s_{ij}f_s(y_j)
$$
We compute the \textbf{output bi-equivariant} function $b:\mathbb{R}^{3\times C}\times \mathbb{R}^{3\times C}\to \mathbb{R}^{3\times 3\times C}$ that takes the tensor product of the two inputs for each channel independently and pass it through an input-output bi-equivariant nonlinearity $\phi$ which in our case is:
\begin{align*}
    b(f_v(x_i),f_v(y_{pi}))=\phi(f_v(x_i)\otimes f_v(y_{pi})), \quad
    \phi(F)=\mathrm{LayerN}\left(\lVert F \rVert\right)\frac{F}{\lVert F\rVert}
\end{align*}
where $\otimes$ is the channel-wise tensor product, $\mathrm{LayerN}$ is the LayerNorm \cite{ba2016layer} and $\lVert . \rVert: \mathbb{R}^{3\times 3\times C}\to\mathbb{R}^{C} $ computes the Frobenius norm for each $3\times 3$ matrix.

Finally  to align the equivariant features $f_v(y_{pi})$ so that they rotate according to a rotation of frame $r$ we define the alignment layer $a$ as:
\begin{align}
a(f_v(x_i),f_v(y_{pi}))=b(f_v(x_i),f_v(y_{pi}))f_v(y_{pi})
\label{eq:alignment}
\end{align}

\begin{proposition}\label{prop:align}
The alignment layer is equivariant to the rotations of its first input and invariant to the rotations of its second input:
$a(R_xf_v(x_i),R_yf_v(y_{pi}))=R_xa(f_v(x_i),f_v(y_{pi}))$
\end{proposition}
 Given the set of pairs $(x_i,y_{pi})$ we can define the equivariant cross-attention layer where the query features are the features of points in $x_i\in X$, and the key, value features are the features of points $y_{pi}$ after they have been aligned appropriately so that they rotate according to frame $r$.

In more detail we define the score attention matrix $e_{XY}^\mathrm{pair}$ as:
\begin{align*}
    e_{XY}^\mathrm{pair}(ij)=&\left(f_s(x_i)W_Q\right)\left(f_s(y_{pj})W_K\right)^T+w_qf_v(x_i)^Ta(f_v(x_i),f_v(y_{pj}))w^T_k
\end{align*}
Then assuming that $F_X$,$F_Y$ are sets of invariant and equivariant features of points in $X$, $Y$ respectively, we can define the pair attention as:
\begin{align*}
        \alpha_s^\mathrm{pair}(x_i,F_X,F_Y)=\sum_{x_j\in X}s_{XY}^\mathrm{pair}(ij)W_vf_s(y_{pj}), 
    \alpha_v^\mathrm{pair}(x_i,F_X,F_Y)=\sum_{x_j\in X}s_{XY}^\mathrm{pair}(ij)\left(\mathrm{VN}_V(a(x_j,y_{pj}))\right)
\end{align*}
with $s_{XY}^\mathrm{pair}(ij)$ being the softmax of the attention scores $e_{XY}^\mathrm{pair}(ij)$. 
In Prop. \ref{prop:fini} we prove that $\alpha_s^\mathrm{pair}$ is invariant to the roto-translation of both point clouds $X,Y$. $\alpha_v^\mathrm{pair}$ is equivariant to the roto-translation of $X$ and invariant to the roto-translation of $Y$. 

Here we have defined the attention layer for pairs of the form $(x_i,y_{pi})$, but similarly we can define the symmetric layer for pairs of the form $(y_i,x_{pi})$ that is equivariant to the rotation of point cloud $Y$. The use of pairs allows us to do a single alignment of the equivariant features for the keys and values in the cross attention by computing $a(x_i,y_{pi})$ once. Then after the alignment we can perform a regular cross-attention which is much more computationally and memory efficient than having to compute $|X|*|Y|$ alignments for all possible combinations of points for $|X|$ and $|Y|$.
\vspace{-10pt}
\subsection{Coarse point correspondence}\label{sec:CoarseCorr}
\vskip -0.2cm
For the estimation of the superpoint matches we utilize the equivariant backbone presented in Section \ref{sec:featExtra}, followed by a coarse correspondence model that iteratively applies intra-point self-attention between the points of the same point cloud and inter-point cross attention between the points of both point clouds. For the intra-point self-attention we are using in parallel the invariant and equivariant self-attention layers presented above. For the inter-point cross attention we used a composition of a simple cross-attention layer only between the invariant features of the two point clouds, followed by an equivariant cross-attention layer defined above.

\par The input of the coarse correspondence transformer is the per-point \emph{invariant} and \emph{equivariant} features extracted by the backbone for the superpoints $X_S$, $Y_S$. Its outputs are invariant per superpoint features for both point clouds, namely $f_{cx},f_{cy}$ for all $x\in X_S$, $y\in Y_S$. The extracted features are then used to compute a correlation matrix $S$ between all the superpoints of the reference and the source. After extracting the correlation between the coarse superpoints we select the top-K entries of $S$ as the candidate superpoint matches. These matches will be invariant to roto-translations of the input point-cloud, since the features used for their computation are invariant.  
\subsection{Fine point matching}\label{sec:fineMatching}
\vskip -0.2cm
Given a candidate pair of matched superpoints $(x_{k(n)},y_{k(n)})$ we perform fine point matching on their corresponding local neighborhoods $\mathcal{N}_{x_{k(n)}}\subseteq X_D$, $\mathcal{N}_{y_{k(n)}}\subseteq Y_D$. We define the neighborhood $\mathcal{N}_{x_{k(n)}}\subseteq X_D$ as the set of all the fine points that have $x_{k(n)}$ as their closest coarse point $\mathcal{N}_{x_{k(n)}}=\left\{x\in X_D|x_{k(n)}=\underset{x_j\in X_S}{\arg\min}(\lVert x-x_j\rVert)\right\}$, and similarly for $\mathcal{N}_{y_k}$. The dense point correspondences are extracted using an optimal transport layer with a cost matrix defined as     $C_k=\left(F_{x_k}F_{y_k}^T\right)/\sqrt{d}$. Here $F_{x_k}\in R^{C\times |\mathcal{N}_{x_k}|}$, $F_{y_k}\in R^{C\times |\mathcal{N}_{x_k}|}$ are matrices with columns containing scalar features for each point of the corresponding local neighborhoods. Similar to the coarse matches, in order for the optimal transport cost and consequently the assignment of the fine point matches to be invariant to rigid transformation, the features represented as columns of $F_{x_k}$, $F_{y_k}$ should also be invariant to these transformations. \\
Since we require  $F_{x_k}$, $F_{y_k}$ to contain invariant features, we can include the equivariant vector features $f_v(p)$ of a fine point $p\in X_{(1)}$ or $p\in Y_{(1)}$, by defining the invariant feature
$U_p=W\text{vec}(f_v^Tf_v)$, 
where  $W$ is a learnable matrix that mixes the elements of $f_v^Tf_v$. From Eq. \ref{eq:EqFeat} we can easily observe how $U_p$ is invariant since the inner product $f_\mathrm{v}^Tf_\mathrm{v}$ between two equivariant features remain invariant under a roto-translation of the input point-cloud. As a result the columns of $F_{x_k}$, $F_{y_k}$ that correspond to individual point features, for points of $p$ of the neighborhoods $\mathcal{N}_{x_k}$, $\mathcal{N}_{y_k}$, are computed by concatenating $U_p$, $f_\mathrm{inv}(p)$ to create a invariant feature $f_p=[U_p^T,f_\mathrm{inv}^{(1)}(p)^T]^T$. 

\par Using the cost matrix $C_k$ of the $k^{th}$ coarse match we utilize the Sinkhorn algorithm  \citep{sinkhorn1967concerning} to compute a matrix $Z_k$, that provides a soft assignment between the fine points of the two neighborhoods. We identify the pair of points for which their corresponding entry is among the top-M entries in both their row and their columns, which we refer to as the mutual top-M set of $Z_k$. This results in a set $M_k$ of fine point matches corresponding to the candidate coarse match $(x_{k(n)},y_{k(n)})$. The final alignment transformation is computed using a local-to-global registration scheme proposed in \cite{qin2022geometric} (See Appendix \ref{sec:loc2glob})
\vspace{-5pt}
\subsection{Iterative Refinement}\label{sec:iterative}
\vskip -0.2cm
Given an initial estimation of the alignment transformation $R_0,T_0$ produced by our model, we can perform a refinement step by iteratively applying our model and using the previous estimated transform as an extra input. To incorporate this additional input, after the first estimation of $R_0,T_0$, we use this estimation to align the equivariant features before the cross attention, which replaces the alignment layer $a(.,.)$ defined in Eq.\ref{eq:alignment}. In the experiments, we present the results of our method when we perform three additional refinement steps.
\vspace{-5pt}
\section{Experiments}\label{sec:experiments}
We evaluate our method on the 3DMatch \cite{zeng20163dmatch} and the challenging 3DLoMatch \cite{Huang_2021_CVPR} datasets which contain scans of indoor scenes with varying levels of overlap. The 3DMatch dataset contains 46 scenes for training, 8 for validation, and 8 for testing. Following the protocol of \cite{Huang_2021_CVPR} we evaluate on the 3DMatch test which contains scenes with an overlap of $30\%$ and above and on the 3DLoMatch test set, which contains scenes with overlap ranging from $10\%$ to $30\%$. For the quantitative evaluation of our method we use similar metrics to previous works \cite{qin2022geometric, Huang_2021_CVPR} (see Appendix \ref{sec:evalmetr} for more details).

\subsection{Robustness Analysis to the initial pose of the point clouds}\label{sec:Rob}
\begin{figure*}[h]
    \centering
     \includegraphics[width=0.32\textwidth]{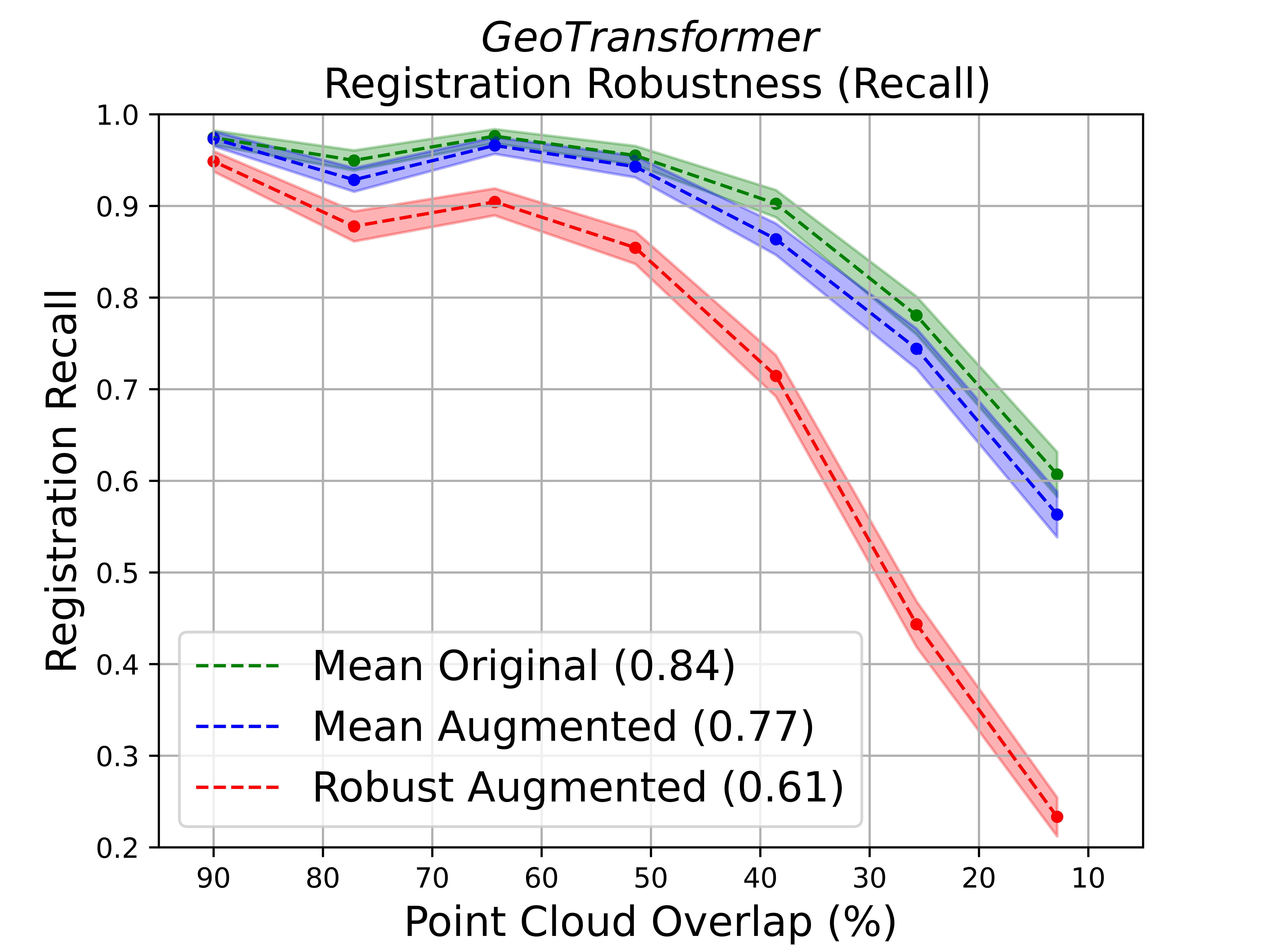}
     \includegraphics[width=0.32\textwidth]{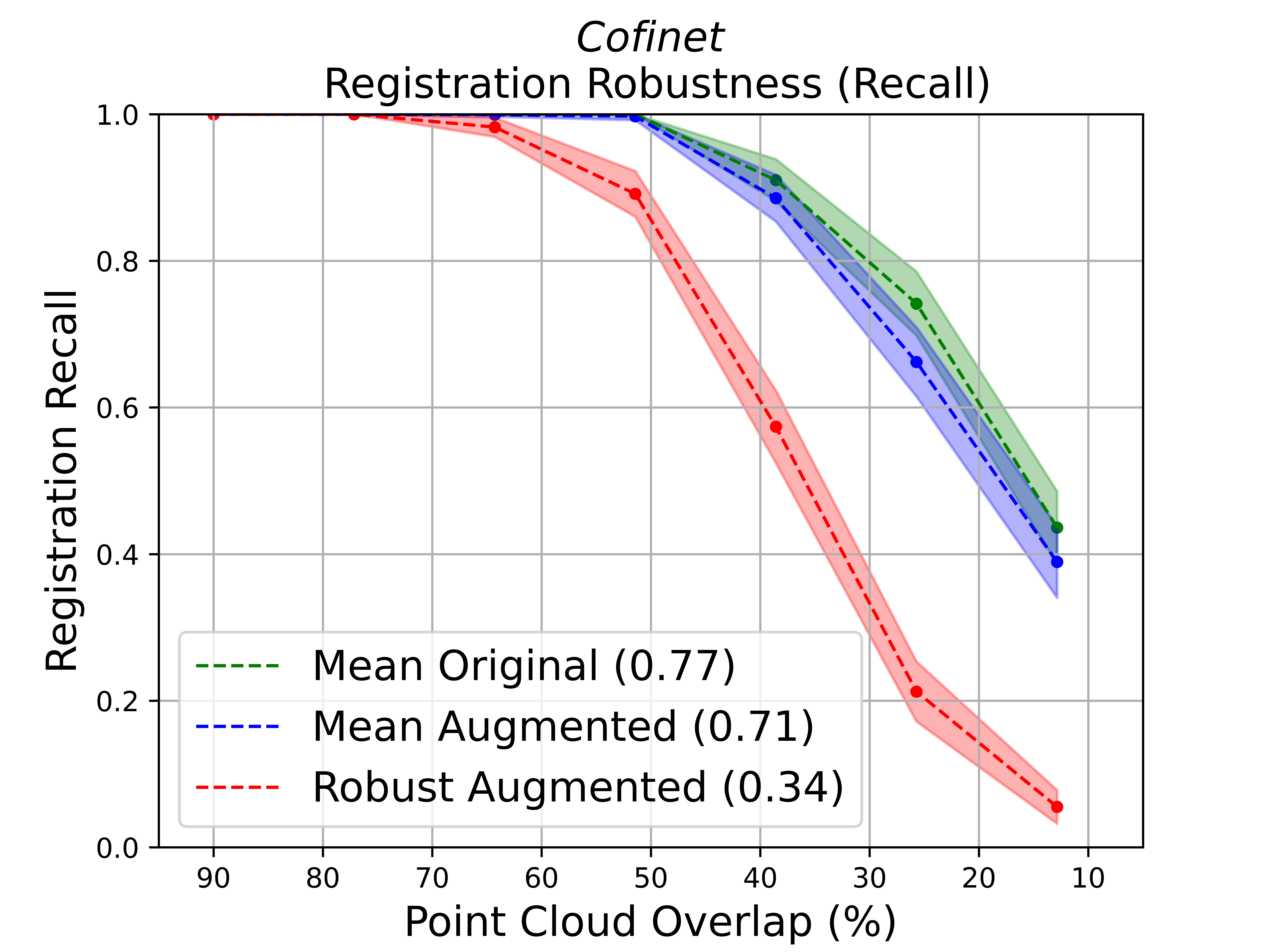}
      \includegraphics[width=0.32\textwidth]{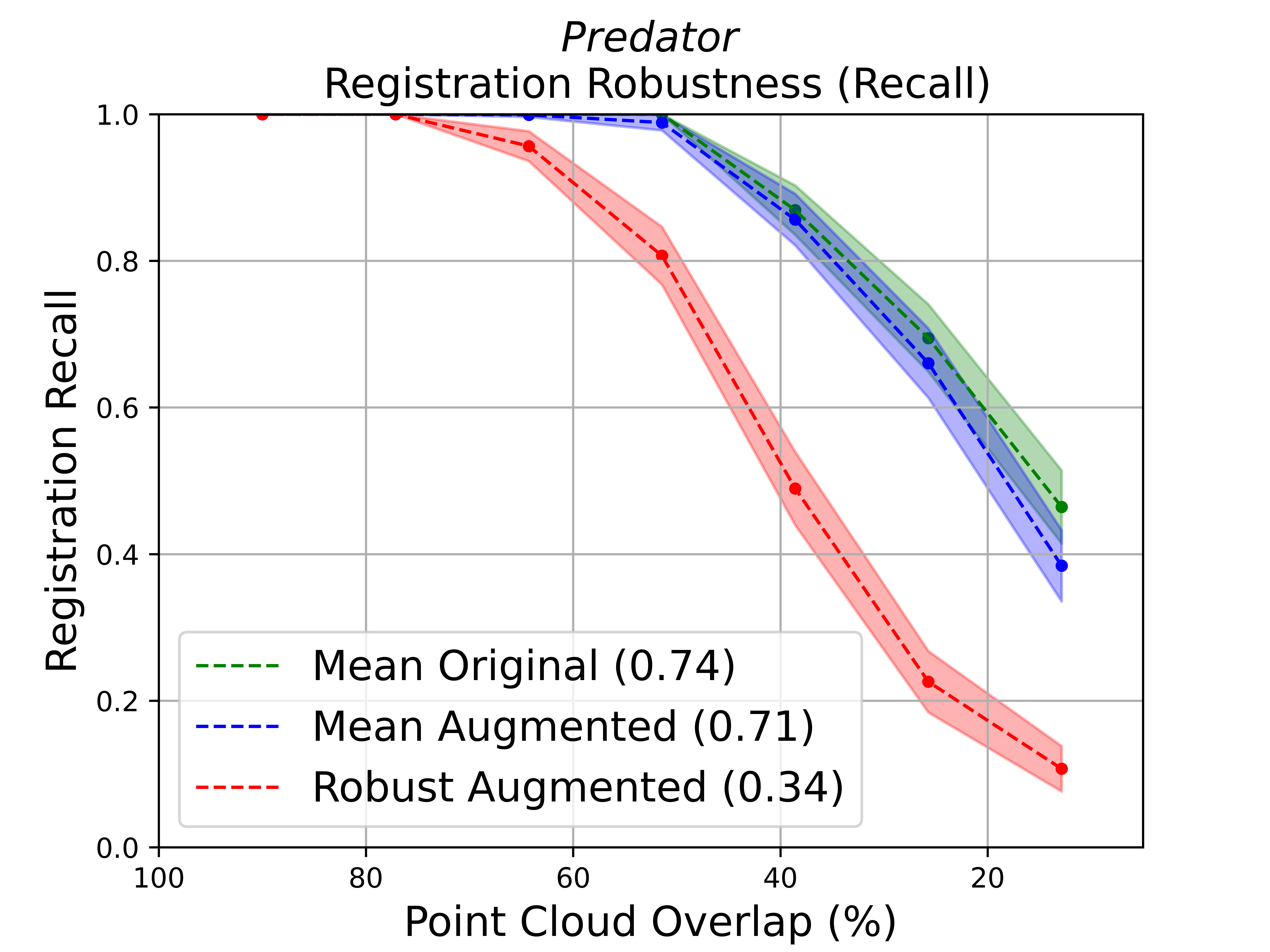}
    \includegraphics[width=0.32\textwidth]{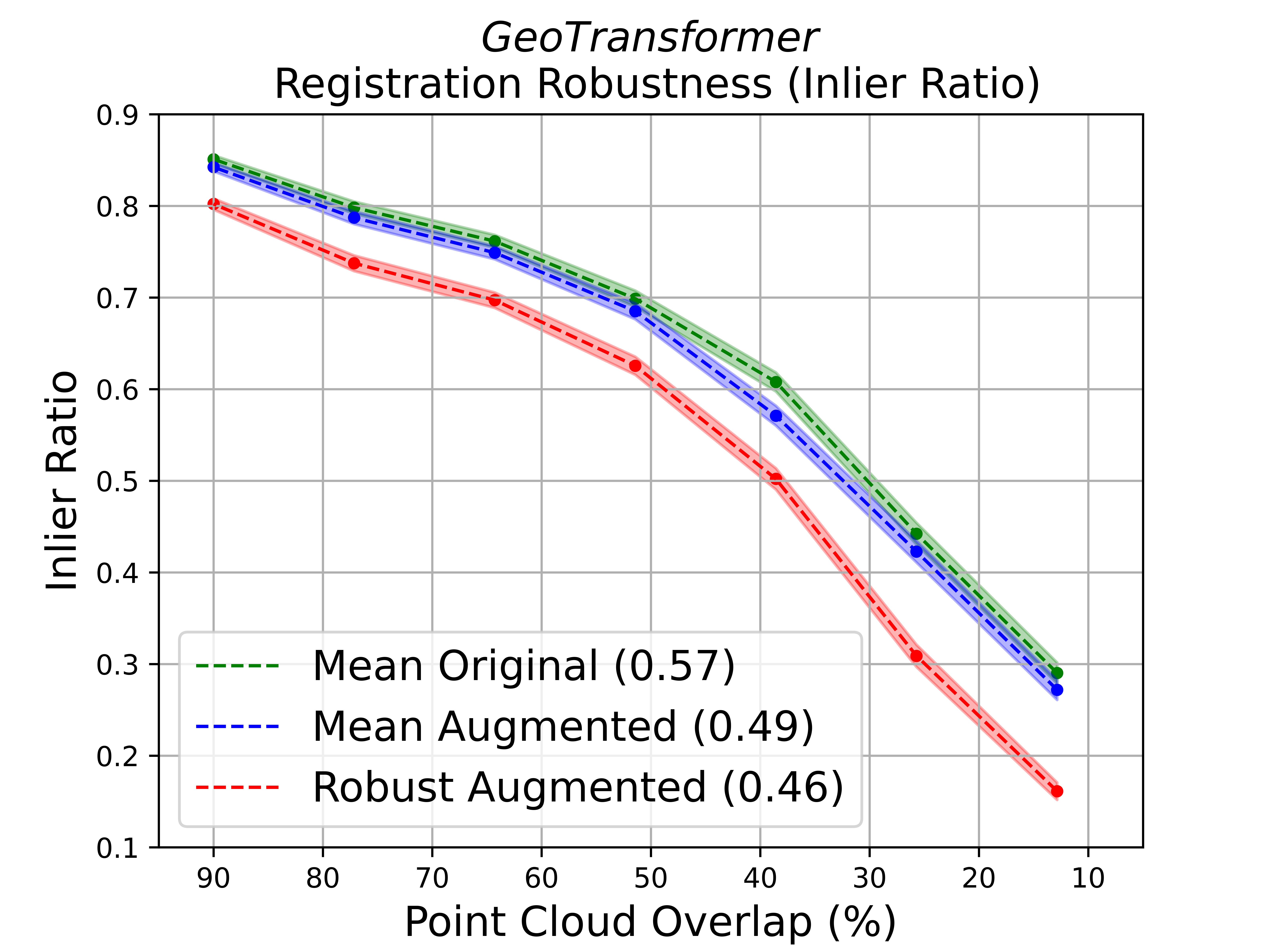}
    \includegraphics[width=0.32\textwidth]{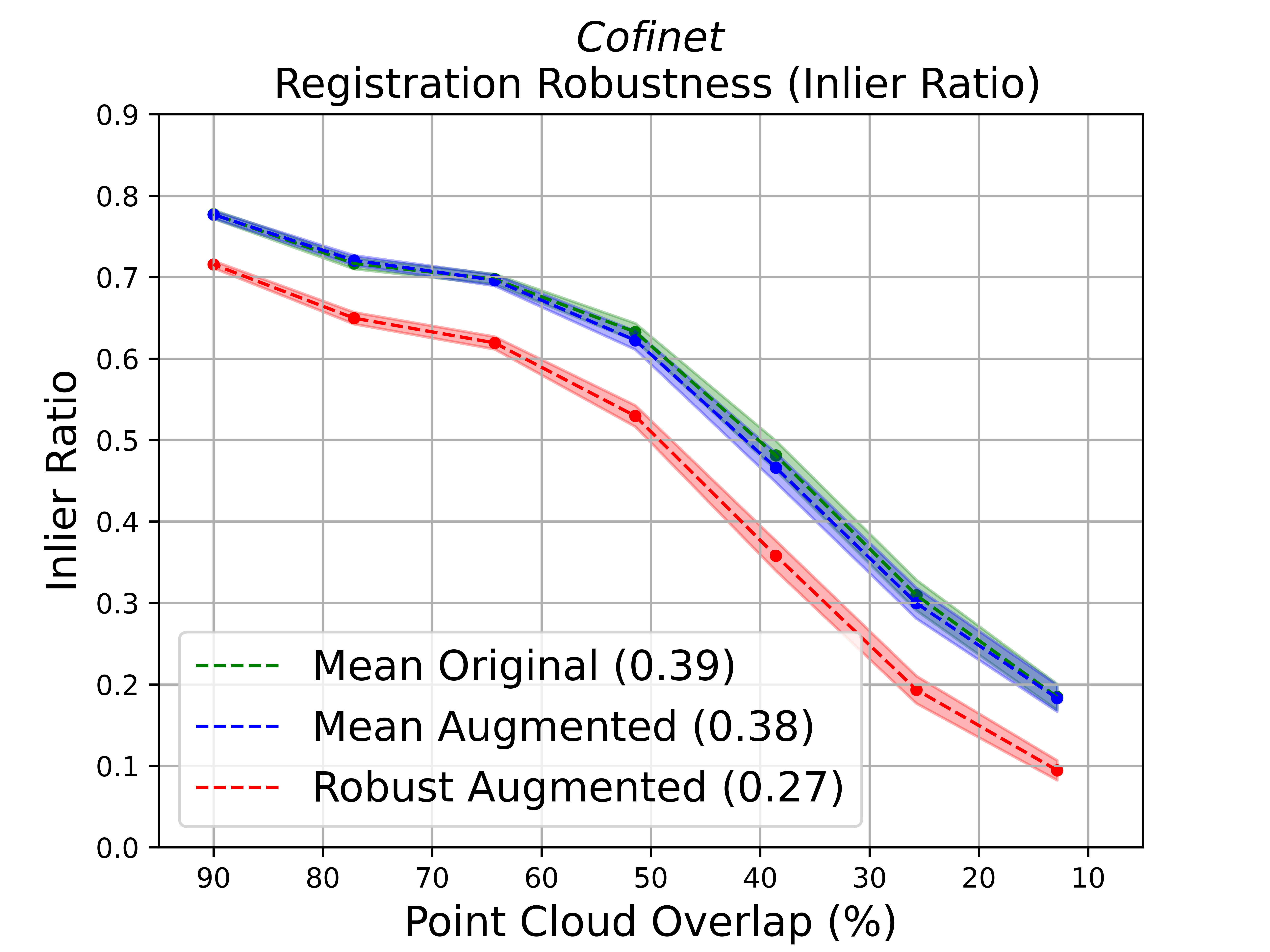}
    \includegraphics[width=0.32\textwidth]{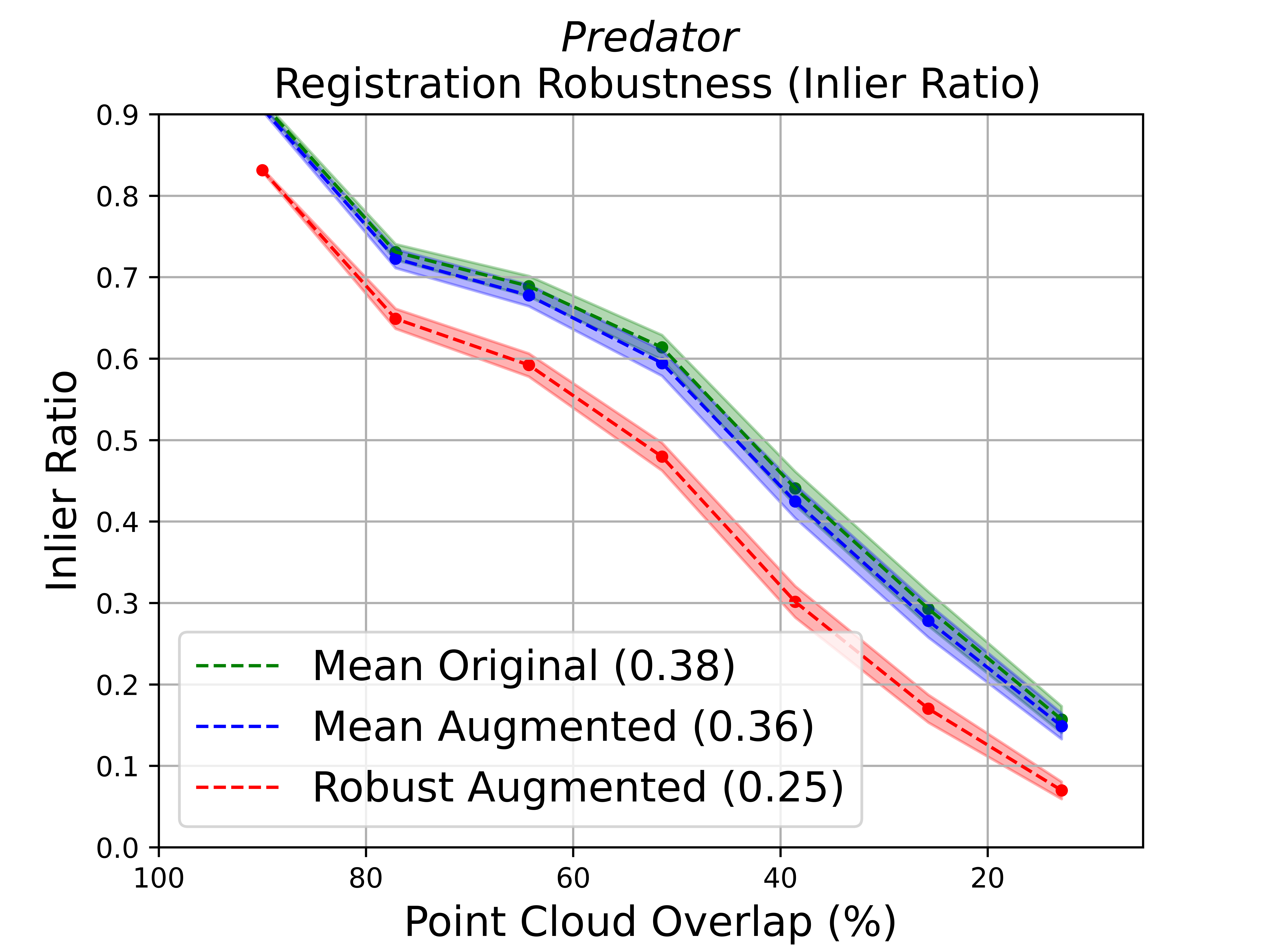}
    \caption{Registration Recall and Inlier Ratio for GeoTransformer \cite{qin2022geometric}, Cofinet \cite{yu2021cofinet} and Predator \cite{Huang_2021_CVPR} on different overlap ranges of the total 3DMatch \cite{zeng20163dmatch}. The green lines (mean original) show the mean per overlap range for the original dataset. The blue lines (mean augmented) show the mean per overlap range of an augmented dataset in which each point cloud has been uniformly roto-translated creating a total of 54 configurations per pair. The red line (robust augmented) shows the mean per overlap range of the minimum across the 54 different configurations. The total mean across all pairs in the dataset for each case is also shown in the plot.}
    \label{fig:brokenall}
\end{figure*}
We benchmark popular state-of-the-art methods \citet{qin2022geometric,yu2021cofinet,Huang_2021_CVPR} on their robustness to the initial poses of the scans. We test all methods in the total 3DMatch dataset \cite{zeng20163dmatch} by concatenating the 3DMatch and 3DLoMatch splits and test the mean performance across different overlap intervals. In Fig.\ref{fig:brokenall} we plot the Registration Recall and the Inlier Ratio in 3 different settings. First, in the green lines (mean original), we show the mean performance of the methods in each overlap interval in the original dataset. The overlap of each pair is calculated as in \cite{Huang_2021_CVPR} in the ground truth registration. Second, in the blue lines (mean augmented), we show the mean performance in an augmented dataset where each point cloud from each pair has been individually rotated around 9 axes uniformly selected and with 3 different angles around each axis also uniformly selected. Thus, from each pair, we create 54 configurations. Lastly, in the red lines (robust augmented), we show the robust loss i.e., the mean performance for each overlap region of the minimum performance across the 54 different configurations of the same pair. The total mean across all pairs in the dataset for each case is also shown in the figure. 

We observe that there is a big drop in performance in the augmented setting both in the average ($6-7\%$) and in the robust ($23-43\%$) metrics, which is exacerbated as the overlap of the point clouds becomes smaller. This is indicated by the fact that the difference between the lines increases as the overlap decreases in the Registration Recall in Fig. \ref{fig:brokenall}. We also observe that GeoTransformer is more robust to initial poses than the rest of the methods which is attributed to the invariant design of the transformer part that learns to match the superpoints between the point clouds. The reason that the method still performs erratically in different initial poses is that the backbone, KP-Conv \cite{Thomas2019KPConv}, is not rotation equivariant. From this observation, we conclude that baking-in equivariance even in parts of the pipeline can be beneficial for global PCR. A visual example of such inconsistent registration is shown in Figure \ref{fig:brokenvisual} where Geotransformer is able to correctly register a pair of point clouds in one configuration but fails to do so in a different configuration. These observations indicate that the problem of \textit{global} PCR remains unsolved and there is a need for a pipeline that performs consistently, irrespective of the poses of the point clouds. On the other hand, our method is designed to consistently register the given scene in all possible configurations of the input pose, since it is bi-equivariant to rigid transformations of the inputs. 

\subsection{Quantitative Comparison} 
\vskip -0.3cm
We  compare the performance of our method against recent state-of-the-art, FCGF~\cite{FCGF2019}, D3Feat~\cite{bai2020d3feat}, SpinNet~\cite{ao2020SpinNet}, Predator~\cite{Huang_2021_CVPR},YOHO ~\cite{wang2022you}, CoFiNet~\cite{yu2021cofinet}, GeoTransformer ~\cite{qin2022geometric}. All methods are trained on the training set of 3DMatch and are evaluated in both 3DMatch and 3DLoMatch. All methods are trained with rotation augmentations for both the source and reference point clouds. In Table \ref{table:RResults} we present the Registration Recall separately for the original 3DMatch and 3DLoMatch. Then, in order to measure robustness to the initial poses of the point clouds, which is the important metric for \textit{global} PCR, we estimate the \textit{expected registration recall} (Mean RR) across different initial poses and the \textit{robust registration recall} which is the average over the dataset of the minimum recall over different poses of the input. To estimate these metrics we create an augmented test dataset where in each pair of point clouds we apply 3d rotation around 9 axes uniformly selected and around 3 angles per axis for both the source and reference point clouds. Thus for each pair, we create 54 configurations and we report the metrics on this augmented dataset. 

We observe that our method achieves comparable results with other state-of-the-art methods in the canonical test set, being second only to GeoTransformer. Moreover, it achieves state-of-the-art performance in the expected and robust metrics. This validates the argument that our bi-equivariant design is an important step towards \textit{global} PCR without sacrificing performance on the canonical setting. Visualizations of low-overlap registrations are provided in Appendix Fig. \ref{fig:qualResult}. 

In Table \ref{table:ablation} we provide an ablation study to show the importance of the proposed bi-equivariant layers as well as the proposed equivariant iterative refinement. First, we provide a simple bi-equivariant alternative to GeoTransformer by replacing the non-equivariant feature extractor KPConv \cite{Thomas2019KPConv} with the equivariant VNN \cite{deng2021vn}. We show that BiEquiFormer, which in addition uses bi-equivariant layers that fuse the information from the two point clouds demonstrates improved performance on the task. Moreover, we experimented with local refinement steps after the initial global alignment. We ran the non-equivariant ICP algorithm, heavily tuned (Point-to-Plane ICP with Robust loss \cite{Pomerleau}). Then we ran the equivariant iterative scheme described in Section \ref{sec:iterative}. In this case too, our method yields better results.

\begin{table}[!t]
\vskip -2cm
\setlength{\tabcolsep}{5pt}
\scriptsize
\centering
\begin{tabular}{l|cc|c|c|c|c}
\toprule
 & \multicolumn{2}{c|}{\textbf{Canonical}} & \multicolumn{4}{c}{\textbf{Roto-translated}} \\
\hline
\multirow{2}{*}{Model} & \multicolumn{2}{c|}{RR} & \multicolumn{1}{c}{Mean RR} &  \multicolumn{1}{c}{Robust RR} & \multicolumn{1}{c}{Mean IR} & \multicolumn{1}{c}{Robust IR}\\
 & 3DM & 3DLM & 3DM+3DLM & 3DM+3DLM & 3DM+3DLM & 3DM+3DLM\\
\midrule
FCGF \cite{FCGF2019} & 0.85 & 0.40 & - & - & - & -\\
D3Feat \cite{bai2020d3feat} & 0.82 & 0.37 & - & - & - & - \\
Predator \cite{Huang_2021_CVPR} & 0.89 & 0.60  & 0.71 & 0.34 & 0.36 & 0.25\\
CoFiNet  \cite{yu2021cofinet}&
0.89 & 0.68 & 0.71 & 0.34 & 0.38 & 0.27 \\
GeoTransformer \cite{qin2022geometric} & \cellcolor{tabthird}0.91 & \cellcolor{tabfirst}0.74 & \cellcolor{tabthird}0.77 &0.61 & \cellcolor{tabfirst}0.49 & 0.46\\
Lepard\cite{lepard2021} & \cellcolor{tabfirst}0.92 & 0.65 & 0.64 & 0.60 & 0.37 & 0.30\\
\hline
SpinNet \cite{ao2020SpinNet} & 0.89 & 0.60 & 0.72 & - & 0.36 & -\\

YOHO \cite{wang2022you}& 0.90 & 0.65 & 0.76 & - & 0.43 & - \\
RIGA \cite{RIGA} & 0.89 & 0.65 & \cellcolor{tabthird}0.77 & \cellcolor{tabthird}0.77 & \cellcolor{tabthird}0.47 & \cellcolor{tabthird}0.47\\
\textbf{BiEquiformer} & 0.90 & \cellcolor{tabthird}0.69 & \cellcolor{tabfirst}0.78 & \cellcolor{tabfirst}0.78 & \cellcolor{tabfirst}0.49 & \cellcolor{tabfirst}0.49\\
\bottomrule
\end{tabular}
\caption{Top: Non-equivariant methods, Bottom: Equivariant methods. Canonical Registration Recall (RR) on 3DMatch (3DM) and 3DLoMatch (3DLM), Mean and Robust Registration Recall (Mean RR, Robust RR) and Inlier Ratio (Mean IR, Robust IR) on the total Rotated 3DMatch (concatenation of the 3DMatch and 3DLoMatch) for inputs augmented by uniform rotation.}
\end{table}

\begin{table}[!t]
\setlength{\tabcolsep}{5pt}
\scriptsize
\centering
\begin{tabular}{l|cc}
\toprule
\multirow{2}{*}{Model} & \multicolumn{2}{c}{RR} \\
 & 3DM & 3DLM  \\
\midrule
VNN+GeoTransformer & 0.87 & 0.62 \\
BiEquiformer + ICP & \cellcolor{tabthird}0.88 & \cellcolor{tabthird}0.66 \\
BiEquiFormer & \cellcolor{tabfirst}0.90 & \cellcolor{tabfirst}0.69 \\
\bottomrule
\end{tabular}
\caption{Ablation study on BiEquiformer. VNN+GeoTransformer replaces the non-equivariant KPConv \cite{Thomas2019KPConv} with an equivariant counterpart VNN \cite{deng2021vn}. BiEquiFormer+ICP utilizes the bi-equivariant layers but refines with a non-bi-equivariant ICP. BiEquiFormer uses the equivariant iterative scheme described in Section \ref{sec:iterative}}\label{table:ablation}
\vspace{-20pt}
\end{table}

\section{Conclusion} 
\vskip -0.3cm
In this work we proposed a novel bi-equivariant pipeline to address the task of \textit{global} PCR i.e. registration without the assumption of a good initial guess of the input point clouds. We investigated the robustness of current deep learning methods on the poses of the input scans and observed a large performance degradation, especially in low-overlap settings. We proposed to address the issue by utilizing equivariant deep learning and formulated and characterized the bi-equivariant properties of PCR. Since standard rotational equivariant layers have large memory overhead but most importantly, they extract features separately from each point cloud, we proposed to build novel, expressive bi-equivariant layers that fuse the information of the two point clouds while extracting per-point features on them. We used those layers to build BiEquiformer a bi-equivariant attention architecture that is scalable to the large volume of points in scene-level scans. We evaluated our method on both the 3DMatch and the challenging 3DLoMatch dataset, showing that our method can achieve comparable and even superior performance to other non-equivariant and equivariant state-of-the-art methods, especially in the robust metrics.

We believe that the explicit formulation and characterization of the bi-equivariance of PCR can be extended to other problems such as pick-and-place tasks in robotic manipulation. We are confident that the bi-equivariant layers that we designed in this work will be beneficial for such tasks too. As a limitation, we pinpoint that while the method achieves state-of-the-art performance in the robust case, there is a small gap in the canonical setting. We believe that this can be attributed to the expressivity of the VNN feature extractor in the first step of the pipeline. However, higher-order steerable feature extractors are currently not scalable to scene-level scans. 

\section*{Acknowledgements}
This project was funded by the grants: ARO MURI W911NF-20-1-0080 and ONR N00014-22-1-2677.

\bibliography{Neurips2024/main}
\bibliographystyle{abbrvnat}

\newpage

\section{Appendix / Supplementary Material}
\subsection{Equivariant Feature Extraction} \label{sec:featExtra}
 Previous works utilize commonly used point cloud processing architectures, such as  KPConv-FPN \citep{Thomas2019KPConv} or DGCNN \citep{dgcnn}, to extract per point features for each point-cloud individually. These features are not inherently designed to be equivariant to rigid transformations.
 We address this limitation by using a backbone feature extractor that outputs both invariant $f_s$ and equivariant $f_v$ feature vectors. Under a roto-translation $R,T$ of the input these features transform as:
 \begin{align}
    f_s(Rx_i+T,RX+T)=f_s(x_i,X), \qquad f_v(Rx_i+T,RX+T)=Rf_v(x_i,X)\label{eq:EqFeat}
\end{align}
To process such equivariant vector features we utilize the Vector Neurons layer proposed in \cite{deng2021vn}. This type of linear layer, denoted as $\mathrm{VN}$, processes features of the form $F\in \mathbb{R}^{3\times C}$, with columns corresponding to vectors in $\mathbb{R}^3$. It is defined as $\mathrm{VN}(F)=FW_{lvn}$, and is equivariant to rotations of its input features since $\mathrm{VN}(RF)=RFW_{lvn}=R\mathrm{VN}(F)$. 

Additionally, to capture the geometry of the scenes at different levels of detail we use a hierarchical architecture, similar to  \citet{chen2022equivariant}, that processes and outputs invariant/equivariant vector features for different subsampled versions of the input point cloud. We denote these subsampled versions as $X_{(0)},X_{(1)},\ldots,X_{(n)}$, ranging from finer to coarser sampled points. We can create the different levels by running for example an equivariant adaptation of Farthest Point Sampling (FPS) where we initialize it from the point closest to the mean. The points in the first downsampling level are referred as dense points $X_D=X_{(1)}$, while the points obtained by the last level of downsampling are referred to as superpoints $X_S=X_{(n)}$. Similarly we use the notation $x_{(i)}$ to distinguish the points for the different subsampling levels.  

\subsection{Implementation Details}\label{sec:impldetails}
\subsubsection{Input pre-processing}
For the initial feature extraction, described in Section \ref{sec:featExtra}, we use four different subsampled versions of the input point cloud, denoted as $X^{(0)},X^{(1)},X^{(2)},X^{(3)}$. Each point cloud is sampled using grid sampling where, for the $i^{th}$ subsampled version $X^{(i)}$, the voxel size is set to $0.025*2^i$.
\\
During training, both the source and the reference point clouds are augmented with Gaussian noise with standard deviation of 0.005.  Additionally, for each point cloud, we limit the total amount of points to 5000. If the input point clouds exceed this limit, we randomly sample 5000 points from each one of them. We observed that enforcing this limit during training has a minimum effect on the performance during testing, even when we test on larger point clouds.
\subsubsection{Model Architecture and Training}
We implemented and evaluated BiEquiFormer in PyTorch \cite{NEURIPS2019_9015} on an I9 Intel CPU, 64GB RAM and an NVIDIA RTX3090 GPU.
\begin{itemize}
    \item \textbf{Feature extraction}: Our feature extraction network consists of consecutive ``hybrid" layers, similar to the ones proposed in \citet{chen2022equivariant}, that simultaneously process both scalar invariant features and equivariant vector features by utilizing Vector Neurons layers \citep{deng2021vn}. In each layer, all points aggregate features from their k nearest neighbors, where we set $k=20$. We perform three aggregation steps for each subsampled version of the point cloud. Similar to KPConv-FPN \citep{Thomas2019KPConv}, we process the different subsampled versions from finer to coarser, where the coarser points have as input features an aggregation of the extracted features of their closest finer points.
    \item \textbf{Coarse point correspondence}: The coarse point correspondence model consists of three consecutive blocks of an intra-point self-attention layer described in Section \ref{sec:intrapoint}, followed by an inter-point cross attention layer that uses only the invariant features of the point clouds, and an equivariant inter-point cross-attention layer described in Section \ref{sec:pairattn}.
    \item \textbf{Fine point matching}: As discussed in Section \ref{sec:fineMatching}, we extract fine point matches between the local neighborhoods of the matched superpoints by using an optimal transport layer. We use the Sinkhorn algorithm \citep{sinkhorn1967concerning} for 100  steps. After extracting the soft assignment between  fine points, we use solve a weighted Procrustes problem, shown in \ref{sec:fineMatching}, to extract the local candidate transformations for the different matched superpoints. Finally, we follow the Local to Global Registration scheme, which selects the candidate transformation that minimizes the total alignment error.
    \item \textbf{Iterative Refinement:}   When we perform the iterative refinement we train an initial model for the first estimation of the alignment transformation and then a second model that performs the refinement steps.
\end{itemize}
 During training we supervise the output of the coarse matching module by using the overlap-aware circle loss proposed in \cite{qin2022geometric}. Additionally, similarly to  \cite{sarlin20superglue} we supervise the fine point matches between the neighborhood $\mathcal{N}_{x_k}$, $\mathcal{N}_{y_k}$ by using a negative log-likelihood loss on the output of the soft assignment matrix $Z_k$ produced by the optimal transport:
\begin{align*}
    \mathcal{L}_{f,k}=-\sum_{(x,y)\in \mathcal{G}_k}\log(z_{x,y})&-\sum_{x\in \mathcal{I}_k}\log(z_{x,m_k+1})\\
    &-\sum_{y\in \mathcal{J}_k}\log(z_{n_i+1,y})
\end{align*}
where $\mathcal{G}_k$ is the set of ground truth fine point matches, $\mathcal{I}_k$, $\mathcal{J}_k$ are the sets containing the rest unmatched points and  $z_{.,m_k+1}$, $z_{n_i+1,.}$ corresponds to the dustbin row and column output from the learnable optimal transport module.
We train our model for 40 epochs, using an initial learning rate of $10^{-4}$ that we reduce by a scale of $0.95$ each epoch. All the parameters are optimized using the Adam optimizer \citep{kingma2015adam}.
\subsection{Local to Global Registration}\label{sec:loc2glob}
The final alignment transformation is computed using a local-to-global registration scheme proposed in \citet{qin2022geometric}. For each candidate coarse match $(x_{k(n)},y_{k(n)})$ and their given set of inliers $M_k$, we compute a candidate transformation $R_i,T_i$ by solving the optimization problem:
\begin{align*}
    \underset{R,T}{\min}\sum_{(p,q)\in M_k}z_{p,q}\lVert Rp+t-q\rVert_2^2
\end{align*}
where $z_{p,q}$ is the entry corresponding to the soft assignment of the fine point $p$ to the point $q$ in the optimal transport matrix $Z_k$. Finally we pick as the global estimated transformation, the candidate that minimizes the alignment error over the combined set of inliers $\bigcup_{k=1,\ldots,M}M_k$.
\subsection{Evaluation Metrics}\label{sec:evalmetr}
 \textbf{Registration Recall (RR)}: the fraction of point clouds whose estimated transformation has an error less by a set threshold. Specifically given a ground truth transformation $P_{gt}$ and the estimated transformation $P_{est}$ we compute the RMSE error:
    \begin{align*}
        \text{RMSE}=\sqrt{\frac{1}{|Y|}\sum_{y\in Y}\lVert P_{gt}^{-1}P_{est}y-y\rVert_2^2}
    \end{align*}
    then the registration recall counts the fraction of registration with RMSE$<0.2$m.
    
    \textbf{Inlier Ratio (IR)} the fraction of fine point correspondences where their residual under the ground-truth transformation is below 0.1m.
    
    \textbf{Relative Rotation and Relative Translation Error}: the relative rotation error and relative translation error between the estimated and ground truth transformation

\subsection{Proofs of Propositions}\label{sec:proofs}
Before beginning with the proofs of the propositions we need to prove a subtle but important point that the joint action is indeed a valid group action of the direct product group.
\begin{proposition}\label{prop:jointaction}
    If the groups $G_1, G_2$ act on the set $S$ via $*,\cdot$ from the right and the left respectively and these actions are \textit{jointly associative} i.e. $(g_1 * s) . g_2 = g_1 * (s . g_2),$ for all $g_1 \in G_1, g_2 \in G_2, s \in S$ then the map defined as:
    \begin{align*}
    &(G_1 \times G_2) \times S \rightarrow S \\
    &((g_1,g_2),s) \mapsto g_1 * (s \cdot g_2^{-1})
    \end{align*}
    is a group action of the direct product group $G_1 \times G_2$.
\end{proposition}
\begin{proof}
    We write the map as $(g_1,g_2)s := g_1 * (s \cdot g_2^{-1})$ for compactness. 
    If $e_1, e_2$ are the identity elements of $G_1,G_2$ then $(e_1,e_2)$ is the identity element of $G_1 \times G_2$. Also consider $(g_1,g_2),(h_1,h_2) \in G_1 \times G_2$ Then,
    \begin{align*}
        1.& (e_1,e_2)s = e_1 * (s \cdot e_2^{-1}) = e_1*(s \cdot e_2) = e_1 * s = s\\
        2.& (g_1,g_2)(h_1,h_2)s = (g_1,g_2)(h_1 * (s \cdot h_2^{-1})) = g_1 * ((h_1 * (s \cdot h_2^{-1})) \cdot g_2^{-1}) \overset{\text{joint assoc.}}{=} \\&= g_1 * (h_1 * ((s \cdot h_2^{-1}) \cdot g_2^{-1})) = g_1 * (h_1 * ((s \cdot h_2^{-1}g_2^{-1}))) = (g_1h_1) * (s \cdot (h_2^{-1}g_2^{-1})) \\& = (g_1h_1,g_2h_2)s 
        \end{align*}
\end{proof}
Due to joint associaticity we can drop the parentheses and write $(g_1,g_2)s := g_1 * s \cdot g_2^{-1}$. We did not do that in the proof to make explicit when the joint associativity was used.

\begin{proof}[Proof of Proposition \ref{prop:pcr_bieq}]
    Given the formulation in Section \ref{sec:PrFormulation} we start by denoting the input point clouds $X^r, Y^s$ and their relative rigid transformation $\mathcal{T}_s^r = \begin{bmatrix}
        R_s^r & T_s^r \\ 0 & 1
    \end{bmatrix}$.
    Also let $C=\{(x_i,y_i)|x_i\in X^r,y_i\in Y^s\}$ denote the point matches. Now, if the input point clouds transform with $\mathcal{T}_1,\mathcal{T}_2 \in \mathrm{SE(3)}$ as: $\mathcal{T}_1 X^r = R_1 X^r + T_1$, $\mathcal{T}_2 Y^s = R_2 Y^s + T_2$ then we need to prove the following for the transformation $\mathcal{T}_1\mathcal{T}_s^r\mathcal{T}_2^{-1} \in \mathrm{SE(3)}$:
    \begin{itemize}
        \item Invariant point matching: The points $\mathcal{T}_1 x_i = R_1x_i+T_1 \in \mathcal{T}_1 X^r,\mathcal{T}_2 y_i = R_2y_i+T_2 \in \mathcal{T}_2 Y^s$ are also point matches for $\mathcal{T}_1\mathcal{T}_s^r\mathcal{T}_2^{-1}$ (which can also be computed from the first problem formulation) since in the new alignment we have: 
        $\mathcal{T}_1\mathcal{T}_s^r\mathcal{T}_2^{-1} (\mathcal{T}_2 y_i) = \mathcal{T}_1\mathcal{T}_s^r y_i$ and
        \begin{align*}
        \|\mathcal{T}_1 x_i - \mathcal{T}_1\mathcal{T}_s^r y_i\|_2 &= \|({R}_1 x_i +T_1)- ({R}_1(\mathcal{T}_s^r y_i) + T_1)\|_2 \\
        &= \|{R}_1(x_i- \mathcal{T}_s^r y_i)\|_2 = 
        \|x_i- \mathcal{T}_s^r y_i\|_2 \leq \epsilon 
        \end{align*}
        since $(x_i,y_i) \in C$.
        \item Optimal Procrustes: For the initial problem we know that the objective function $L_1(\mathcal{T}) =\sum_{(x_i,y_i) \in C}\|\mathcal{T}y_i - x_i\|_2^2$ satisfies: $L_1(\mathcal{T}_s^r):= L_1^* \leq L_1(\mathcal{T})$ for all $\mathcal{T} \in \mathrm{SE(3)}$. Now we look at the objective of the new problem (for which we proved invariant matches) $L_2(\mathcal{T}) = \sum_{(x_i,y_i) \in C}\|\mathcal{T} \mathcal{T}_2y_i - \mathcal{T}_1x_i\|_2^2$. If we substitute $\mathcal{T}=\mathcal{T}_1\mathcal{T}_s^r \mathcal{T}_2^{-1}$ we get:
        $$L_2(\mathcal{T}_1\mathcal{T}_s^r \mathcal{T}_2^{-1}) = \sum_{(x_i,y_i) \in C}\|\mathcal{T}_1\mathcal{T}_s^r \mathcal{T}_2^{-1} \mathcal{T}_2y_i - \mathcal{T}_1x_i\|_2^2 = \sum_{(x_i,y_i) \in C}\|\mathcal{T}_s^r y_i - x_i\|_2^2 = L_1(\mathcal{T}_s^r)=L_1^*$$
        we proved that the optimal of the second problem is upper bounded by the first. We will also show the opposite. In particular, if we substitute $\mathcal{T}=\mathcal{T}_1^{-1}\mathcal{T}\mathcal{T}_2$  in $L_1$ for any $\mathcal{T} \in \mathrm{SE(3)}$ we get:
        \begin{align*}
            L_1(\mathcal{T}_1^{-1}\mathcal{T}\mathcal{T}_2) &= \sum_{(x_i,y_i) \in C}\|\mathcal{T}_1^{-1}\mathcal{T}\mathcal{T}_2y_i - x_i\|_2^2 \\ &= \sum_{(x_i,y_i) \in C}\|R_1^T(\mathcal{T}\mathcal{T}_2y_i)-R_1^T T_1 - x_i\|_2^2 \\&= \sum_{(x_i,y_i) \in C}\|\mathcal{T}\mathcal{T}_2y_i-R_1(R_1^T T_1 + x_i)\|_2^2 \\&=  \sum_{(x_i,y_i) \in C}\|\mathcal{T}\mathcal{T}_2y_i-\mathcal{T}_1 x_i\|_2^2 = L_2(\mathcal{T})
        \end{align*}
    \end{itemize}
\end{proof}
\begin{proof}[Proof of proposition \ref{prop:eq_flip}]
First, we can again prove invariant matching. The flip is a unitary operation so it does not change the distances between the matched points. In other words since $\|x_m-y_m\|_2=\|y_m-x_m\|_2$ the set $C$ of point matches consists of the same points (reversed). Again looking at the two objectives we can prove Procrustes optimality as for $\mathcal{T} \in \mathrm{SE(3)}$ it holds $\mathcal{T}^{-1} \in \mathrm{SE(3)}$:
\begin{align*}
    L_1(\mathcal{T}^{-1}) &= \sum_{(x_i,y_i)\in C}\|\mathcal{T}^{-1}y_i-x_i\|_2^2 \\ &= \sum_{(x_i,y_i)\in C}\|R^Ty_i-R^TT-x_i\|_2^2 \\ &= \sum_{(x_i,y_i)\in C}\|y_i-T-Rx_i\|_2^2  = \sum_{(x_i,y_i)\in C}\|\mathcal{T}x_i-y_i\|_2^2 = L_2(\mathcal{T})
\end{align*}
Thus, the optimal values of the two problems are again there same and since $\mathcal{T}_s^r$ is optimal for $L_1$ then $(\mathcal{T}_s^r)^{-1}$ is optimal for $L_2$.
Lastly, this is indeed an action of the flips since $f^2=e$ and $((\mathcal{T}_s^r)^{-1})^{-1} = \mathcal{T}_s^r$
\end{proof}
\begin{proof}[proof of Proposition \ref{prop:eq_ord}]
Since the permutations is a unitary transformation the distance again as above do not change and the matching is again invariant (this time the set has exactly the same points in some order). Since the sum is order-invariant the value of the objective is also the same so the problem is invariant to point permutations.   
\end{proof}
\begin{proof}[proof of proposition \ref{prop:bifeats}]
    \begin{enumerate}
        \item Since $f_1 \mapsto R_1 f_1, f_2 \mapsto R_2 f_2$ the tensor product $f_1f_2^T \mapsto (R_1f_1)(R_2f_2)^T = R_1(f_1f_2^T)R_2^T$. Thus, the map is output bi-equivariant.
        \item Since all singular values are distinct and positive, we can sort them in $\Sigma = diag\{\sigma_1,\sigma_2,\sigma_3\}$ in which case it is known that the SVD of $F=U\Sigma V^T$ is unique up to a simultaneous sign flip of the columns of $U,V$ i.e., there are 8 choices for $U= [\pm u_1 ~ \pm u_2 \pm u_3]$ and the corresponding for $V$. However, if $\det(F)>0$ then we can select both $U,V \in SO(3)$ i.e. $u_3=u_1 \times u_2$ and $v_3 = v_1 \times v_2$ and if $\det(F) <0$ we can select $U \in O(3)-SO(3), V \in SO(3)$ i.e. $u_3 = - u_1 \times u_2, v_3=v_1 \times v_2$. Since all singular values are positive the determinant cannot be zero. 
        
        That leaves 4 choices i.e. if $(u_1,v_1)$ and $(u_2,v_2)$ are the first and second columns of $U,V$ then $\pm(u_1,v_1)$, $\pm (u_2,v_2)$ are the rest of the choices for the first and second column of $U,V$ which create the valid SVD solutions.
        
        Now,  $R_1F R_2^T = R_1(U\Sigma V^T)R_2^T = (R_1U) \Sigma (R_2V)^T$ and thus $(R_1U, \Sigma, R_2V)$ is an SVD of $R_1F R_2^T$ since the composition of rotation matrix with a unitary matrix is a unitary matrix. Moreover, $\det(R_1FR_2^T) = \det(R_1)\det(F)\det(R_2)=\det(F)$ so if $\det(F)>0$ then $R_1U, R_2V \in SO(3)$ and if $\det(F)<0$ then $R_1U \in O(3)-SO(3), R_2V \in SO(3)$ as is the case for $U,V$.

        So if the set $\{(U_1,V_1),(U_2,V_2),(U_3,V_3),(U_4,V_4)\}$ is the set of valid $U,V$ in the SVD for $F$ then for $R_1FR_2^T$ the corresponding set is: $\{(R_1U_1,R_2V_1),(R_1U_2,R_2V_2),(R_1U_3,R_2V_3),(R_1U_4,R_2V_4)\}$. Also, $\Sigma$ is invariant. Thus we can use any point-wise non-linearity on $\Sigma$ since this is also invariant. And if we define the action $\cdot$ on the set of 4 matrices as $R\cdot\{U_1,U_2,U_3,U_4\} = \{RU_1,RU_2,RU_3,RU_4\}$ then the map: $$F \mapsto (\{U_i\sigma(\Sigma)\}_{i=1}^4,\{V_i\sigma'(\Sigma)\}_{i=1}^4)$$ satisfies: $$R_1FR_2^T \mapsto (R_1 \cdot \{U_i\sigma(\Sigma)\}_{i=1}^4,R_2 \cdot \{V_i\sigma'(\Sigma)\}_{i=1}^4).$$ Thus the map is input bi-equivariant. 
        \item Since $\|R_1 F R_2^T\| = \|F\|$ we get $R_1FR_2^T \mapsto \sigma(\|R_1FR_2^T\|) \frac{R_1FR_2^T}{\|R_1FR_2^T\|} = R_1 \sigma(\|F\|)\frac{F}{\|F\|}R_2^T$. Thus the map is input-output bi-equivariant. 
    \end{enumerate}
\end{proof}
\begin{proposition} \label{prop:intra_inv}
$\alpha_s^\mathrm{intra}$ is invariant and $\alpha_v^\mathrm{intra}$ is equivariant to the roto-translation of the input point cloud (see proof in Appendix):
\begin{align*}
    \alpha_s^\mathrm{intra}(Rx_i+T,f_s,Rf_v)&=\alpha_s^\mathrm{intra}(x_i,f_s,f_v)\\
     \alpha_v^\mathrm{intra}(Rx_i+T,f_s,Rf_v)&=R\alpha_v^\mathrm{intra}(x_i,f_s,f_v)
\end{align*}
\end{proposition} 
\begin{proof}[Proof sketch of proposition \ref{prop:intra_inv}]
It is easy to show that $e_{ij}$ is invariant to transformations of all the inputs of $a_s^\mathrm{intra}$ since the first term  uses only the invariant $f_s$ features and the invariant $r_{ij}$ geometric embedding introduced in \cite{qin2022geometric}. In the second term a transformation by $R$ results in:
\begin{align*}
    w_q(Rf_v(x_i))^T(Rf_v(x_j))w_k=w_qf_v(x_i)^TR^TRf_v(x_j)w_k=w_qf_v(x_i)^Tf_v(x_j)w_k
\end{align*}
which is also invariant.
As a result $\alpha_s^\mathrm{intra}(x_i,f_s,f_v)$ is invariant since it only depends on $e_{ij}$ and $f_v$ and 
since the $\mathrm{VN}$ layer is equivariant to the rotations:
\begin{align*}
    \alpha_v^\mathrm{intra}(Rx_i+T,f_s,Rf_v)=&=\sum_{x_j\in X}\frac{\exp(e_{ij})}{\sum_{x_j'\in X}\exp(e_{ij'})}\mathrm{VN}_V(Rf_v(x_j))\\
    &=\sum_{x_j\in X}\frac{\exp(e_{ij})}{\sum_{x_j'\in X}\exp(e_{ij'})}R\mathrm{VN}_V(f_v(x_j))\\
    &=R\alpha_v^\mathrm{intra}(x_i,f_s,f_v)
\end{align*}
\end{proof}
\begin{proof}[Proof Sketch of proposition \ref{prop:align}]
Here we use the fact that the the Frobenius norm is invariant to the rotation as a results for the nonlinearity we have that:
\begin{align*}
    \phi(R_xFR_y^T)&=\mathrm{LayerN}\left(\lVert R_xFR_y^T \rVert\right)\frac{R_xFR_y^T}{\lVert R_xFR_y^T\rVert}\\
    &=\mathrm{LayerN}\left(\lVert F \rVert\right)\frac{R_xFR_y^T}{\lVert F\rVert}\\
    &=R_x\phi(F)R_y^T
\end{align*}
Then using the fact tensor product is bi-equivariant it is easy to show that:
\begin{align*}
    b(R_xf_v(x_i),R_yf_v(y_{pi}))&=\phi(R_xf_v(x_i)\otimes R_yf_v(y_{pi}))\\
    &=\phi(R_x(f_v(x_i)\otimes f_v(y_{pi}))R_Y^T)\\
    &=R_x\phi((f_v(x_i)\otimes f_v(y_{pi})))R_Y^T
\end{align*}
and
\begin{align*}
   a(R_xf_v(x_i),R_yf_v(y_{pi}))&= b(R_xf_v(x_i),R_yf_v(y_{pi}))R_yf_v(y_{pi})\\
   &=R_xb(f_v(x_i),f_v(y_{pi}))f_v(y_{pi})\\
   &=R_xa(f_v(x_i),f_v(y_{pi}))
\end{align*}
\end{proof}

\begin{proposition}\label{prop:fini}
    $\alpha_s^\mathrm{pair}$ is invariant to the roto-translation of both point clouds $X,Y$. $\alpha_v^\mathrm{pair}$ is equivariant to the roto-translation of $X$ and invariant to the roto-translation of $Y$ (see proof sketch in Appendix). Specifically given $X'=R_xX+T_x$ and $Y'=R_yY+T_Y$ :
    \begin{align*}
    \alpha_s^\mathrm{pair}(R_xx_i+T_x,F_{X'},F_{Y'})&= \alpha_s^\mathrm{pair}(x_i,F_X,F_Y)       \\
    \alpha_s^\mathrm{pair}(R_xx_i+T_x,F_{X'},F_{Y'})&= R_X\alpha_s^\mathrm{pair}(x_i,F_X,F_Y)
    \end{align*}

\end{proposition}
\begin{proof}[Proof sketch of proposition \ref{prop:fini}]
    Here the layer is similar with the one in proposition 4.1 with different second input being $a(f_v(x_i),f_v(y_{pi}))$ that is equivariant to the transformation of frame $X$. So we can show the equivariance using the same arguments as proposition 4.1
\end{proof}
\subsection{Qualitative Results}
\begin{figure*}
\centering
\includegraphics[width=\linewidth]{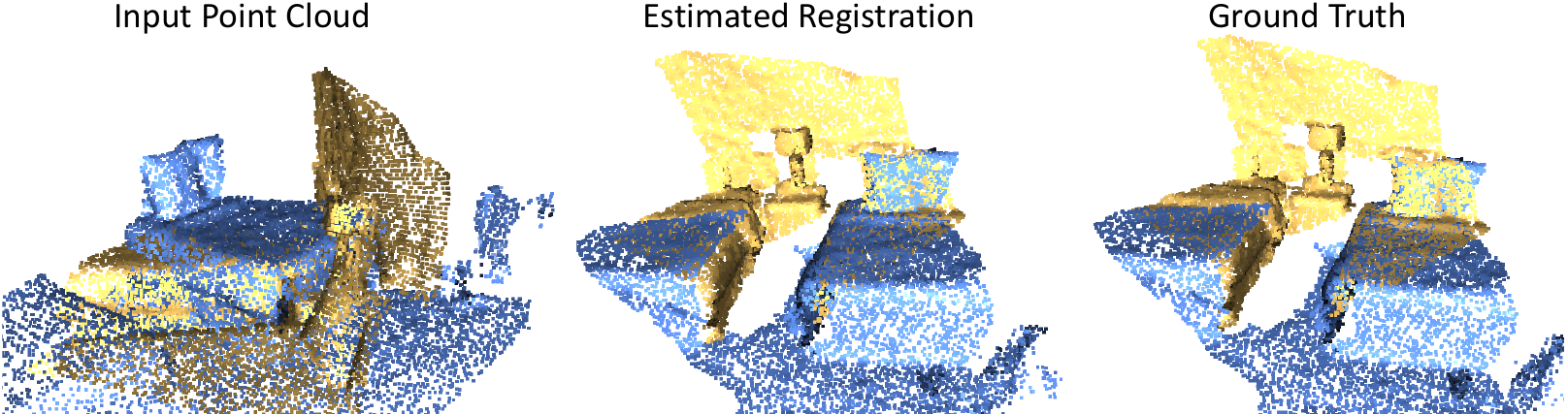}

\includegraphics[width=\linewidth]{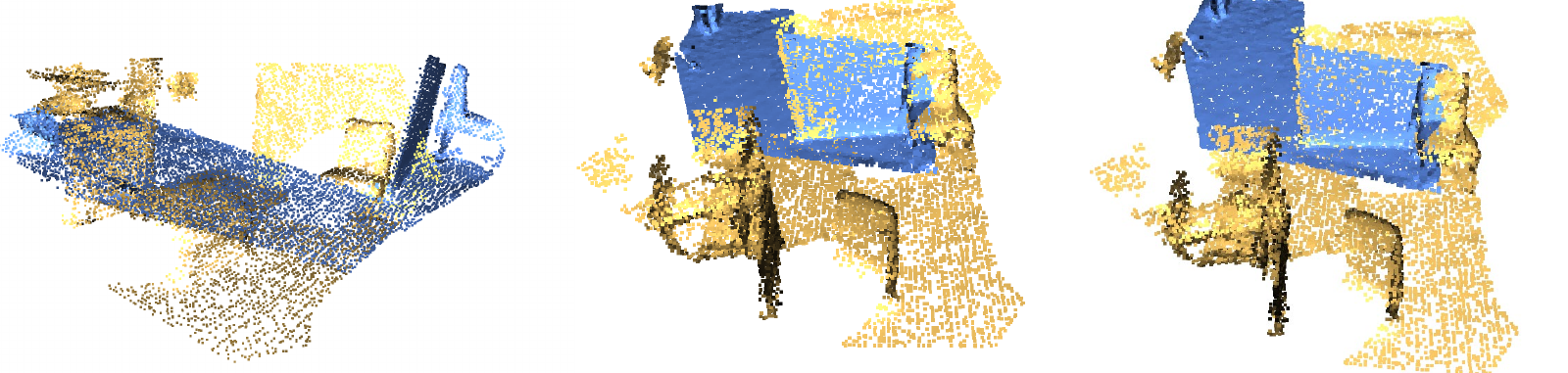}

\includegraphics[width=\linewidth]{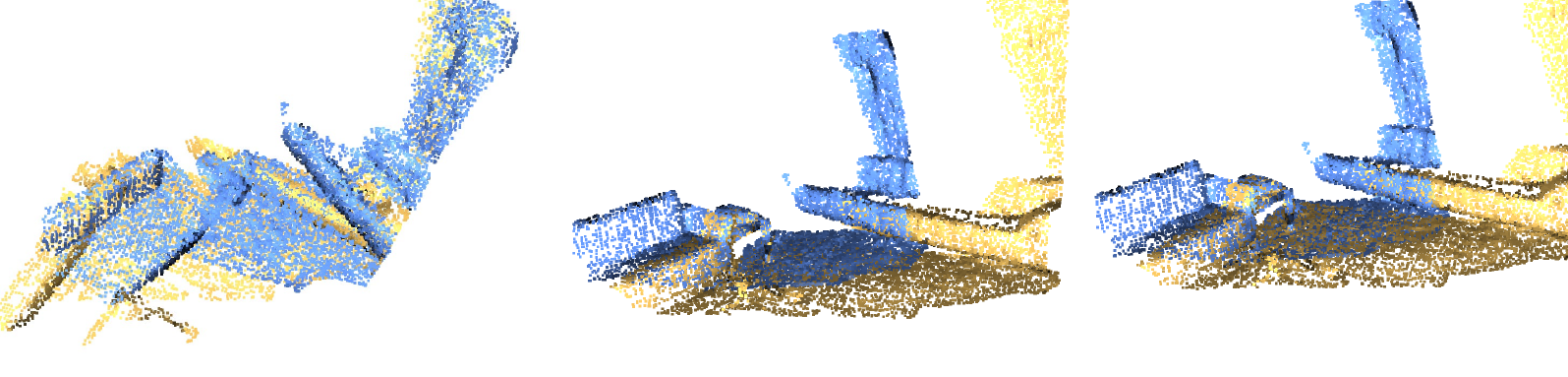}

\includegraphics[width=\linewidth]{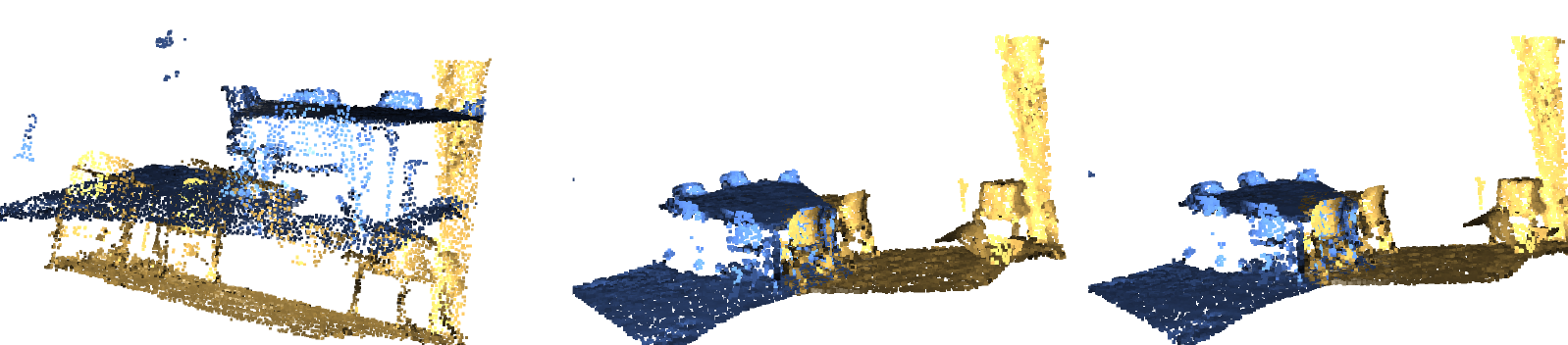}
%

\caption{Registration results achieved by our method compared to the ground truth alignment.}
\label{fig:qualResult}
\end{figure*}
In Figure \ref{fig:qualResult} we provide additional qualitative results with registrations achieved by our method. We show examples of both high and low overlap from the test set of 3DMatch and 3DLoMatch. 
\subsection{Limitations}\label{sec:limitations}
One limitation of the current network is that, while in the robust setting, it achieves state-of-the-art results, in the canonical setting there is a performance gap with the current best methods. We conjecture that this can be attributed to the feature extraction backbone VNN \cite{deng2021vn} and we will investigate alternatives in the future.

Another limitation of the pipeline is an additional memory overhead coming from the tensor products in the attention modules. While we did our best to create a scalable and compact architecture, the toll to satisfy the equivariance constraint exactly is that some blocks might require additional operations to their non-equivariant counterparts. While in the 3DMatch setting, this did not make a difference, the method has to be adapted properly in order to register scenes with millions of points. 

A general limitation of correspondence-based methods like ours is that when the overlap is zero as in Point Cloud Assembly tasks the network cannot treat PCR properly. Moreover, as typical in PCR literature, it is implicitly assumed that there is a correct alignment for the input pairs. The network is designed to predict the best alignment possible even when no alignment is correct. Thus in order to integrate it into bigger SLAM pipelines for loop closure detection etc. additional extensions need to be done.  

Lastly, the case of symmetric parts where multiple alignments are possible is not treated in this work. However, we conjecture that the advantages of our method in  equivariant feature extraction from the neighboorhoods together with the local robust estimators (LGR) that propose different rotations per-neighboorhood before selecting a single one can lead to multiple consistent hypotheses in the cases of symmetric objects.

\subsection{Broader Impact}\label{sec:broader}
In this work, we address a major robustness limitation of current deep learning methods on point cloud registration. Our theoretical and methodological contributions, for example the novel bi-equivariant layers presented, have the potential to advance any pipeline that respects similar symmetries (for example pick-and-place in robotics manipulation). 

Moreover, Point Cloud Registration can be used as the front end of larger SLAM pipelines. Our method guarantees that the registration will be consistent w.r.t. the scan poses meaning that there is no adversarial pose that would make the network behave erratically. If PCR is integrated into safety-critical applications this is a major advancement on verifiable safety. 

~\newpage

\appendix

\end{document}